\documentclass[10pt]{article}
\usepackage{amsmath,graphicx}
\usepackage{amsthm}
\newtheorem{theorem}{Theorem}
\usepackage{algorithm,algpseudocode}
\newtheorem{prop}{Proposition}
\usepackage{url}

\usepackage{epstopdf}
\usepackage{amssymb}
\usepackage{tikz}
\usepackage{pbox}

\def\tens{\mathcal}
\def\vect{\mathbf}
\def\matr{\mathbf}
\def\mrm{\mathrm}

\newcommand{\tr}{\mathop{\mathrm{tr}}\nolimits}

\newcommand{\ve}{\mathrm{vec}}

\usepackage{cite}
\usepackage{hyperref}

\newcommand{\ag}{\mathsf{AG}}
\newcommand{\agn}{\mathsf{AGN}}
\newcommand{\capeagn}{\mathsf{capeAGN}}
\newcommand{\capepca}{\mathsf{capePCA}}
\newcommand{\cape}{\mathsf{CAPE}}
\newcommand{\avn}{\mathsf{AVN}}
\newcommand{\nonpriv}{\mathsf{non-dp\ pool}}
\newcommand{\nonprivT}{\mathsf{Non-priv.}}
\newcommand{\rand}{\mathsf{Rand.\ vect.}}
\newcommand{\conv}{\mathsf{conv}}
\newcommand{\local}{\mathsf{local}}

\usepackage[normalem]{ulem}
\newcommand\redout{\bgroup\markoverwith{\textcolor{red}{\rule[.5ex]{2pt}{0.4pt}}}\ULon}

\newtheorem{lemma}{Lemma}
\begin{document}

\title{Distributed Differentially-Private Algorithms for Matrix and Tensor Factorization}
\date{}
\author{Hafiz Imtiaz and Anand D. Sarwate\thanks{This work was partially supported by NSF under award CCF-1453432, DARPA and SSC Pacific under contract N66001-15-C-4070, and the NIH under award  1R01DA040487-01A1.}%
\thanks{The authors are with the Department of Electrical and Computer Engineering, Rutgers, The State University of New Jersey, 94 Brett Road, Piscataway, NJ 07302, \texttt{hafiz.imtiaz@rutgers.edu}, \texttt{anand.sarwate@rutgers.edu}}
}

%% The paper headers
%\markboth{IEEE Journal of Selected Topics in Signal Processing,~Vol.~XX, No.~XX, XXXX~2018}%
%{Imtiaz and Sarwate: Distributed Differentially-Private Algorithms for Matrix and Tensor Factorization}

\maketitle

\begin{abstract}
In many signal processing and machine learning applications, datasets containing private information are held at different locations, requiring the development of distributed privacy-preserving algorithms. Tensor and matrix factorizations are key components of many processing pipelines. In the distributed setting, differentially private algorithms suffer because they introduce noise to guarantee privacy. This paper designs new and improved distributed and differentially private algorithms for two popular matrix and tensor factorization methods: principal component analysis (PCA) and orthogonal tensor decomposition (OTD). The new algorithms employ a correlated noise design scheme to alleviate the effects of noise and can achieve the same noise level as the centralized scenario. Experiments on synthetic and real data illustrate the regimes in which the correlated noise allows performance matching with the centralized setting, outperforming previous methods and demonstrating that meaningful utility is possible while guaranteeing differential privacy.
\end{abstract}

%\begin{keywords}
%Differential privacy, distributed orthogonal tensor decomposition, latent variable model, distributed principal component analysis%, correlated noise design.
%\end{keywords}

\section{Introduction}\label{sec:intro}
% differential privacy
Many signal processing and machine learning algorithms involve analyzing private or sensitive data. The outcomes of such algorithms may potentially leak information about individuals present in the dataset. A strong and cryptographically-motivated framework for protection against such information leaks is differential privacy~\cite{dwork2006}. Differential privacy measures privacy risk in terms of the probability of identifying individual data points in a dataset from the results of computations (algorithms) performed on that data.

% distributed setting
In several modern applications the data is distributed over different locations or sites, with each site holding a smaller number of samples. For example, consider neuro-imaging analyses for mental health disorders, in which there are many individual research groups, each with a modest
number of subjects. Learning meaningful population properties or efficient feature representations from high-dimensional functional magnetic resonance imaging (fMRI) data requires a large sample size. Pooling the data at a central location may enable efficient feature learning, but privacy concerns and high communication overhead often prevent sharing the underlying data. Therefore, it is desirable to have efficient distributed algorithms that provide utility close to centralized case and also preserve privacy~\cite{SarwatePTAC:14sharing}. 

% matrix and tensor factorization
This paper focuses on the Singular Value Decomposition (SVD) or Principal Component Analysis (PCA), and orthogonal tensor decompositions.
Despite some limitations, PCA/SVD is one of the most widely-used preprocessing stages in any machine learning algorithm: it projects data onto a lower dimensional subspace spanned by the singular vectors of the second-moment matrix of the data.
%PCA/SVD is used for preprocessing high-dimensional data by projecting it onto a lower dimensional subspace spanned by the singular vectors of the second-moment matrix of the data. For example, training a classifier is much efficient when the data is first projected onto lower dimensions. 
Tensor decomposition is a powerful tool for inference algorithms because it can be used to infer complex dependencies (higher order moments) beyond second-moment methods such as PCA. This is particularly useful in latent variable models~\cite{anandkumar2012} such as mixtures of Gaussians and topic modeling.
%such as Singular Value Decomposition (SVD) or Principal Component Analysis (PCA). SVD or PCA operate only on sample second-moment or covariance matrices. In contrast, tensor decomposition can be employed to exploit the higher-order moments and is shown to be useful for learning latent variable models~\cite{anandkumar2012}. Particularly, orthogonal decomposition of symmetric tensors are effective for topic modeling and learning mixtures of Gaussians. 

%\ads{going to focus on big edits/markups to save space and then do wordsmithing later}
%
%\ads{The story is a little muddled/confusing here -- the main story is to say that CAPE can help with these types of factorization problems, so the necessary additional resource (trusted noise generator) should be front and center so that it says ``if we have this resource then good things will happen.''}
%
%\ads{For CAPE, we can say this requires a small extension to the unequal data size case which is less important for theoretical reasons but very important for practical reasons}

\noindent\textbf{Related Works. } For a complete introduction to the history of tensor decompositions, see the comprehensive survey of Kolda and Bader~\cite{kolda2009} (see also Appendix~\ref{appendix:tensor_preliminaries}).
%The very first tensor decomposition ideas (e.g. tensor rank and polyadic decomposition) are attributed to Hitchcock~\cite{hitchcock1927,hitchcock1928}. Tensor decomposition and multi-way signal models were used in the context of latent variable models in psychometrics~\cite{cattell1944}. It became popular in neuroscience, signal processing and machine learning later. 
The CANDECOMP/PARAFAC, or CP decomposition~\cite{carroll1970,harshman1970} and Tucker decomposition~\cite{tucker1966} are generalizations of the matrix SVD to multi-way arrays.
%The most well-known tensor decomposition algorithms are Canonical Polyadic Decomposition (CANDECOMP) or Parallel Factors (PARAFAC). These two algorithms are often referred together as the CP decomposition~\cite{carroll1970,harshman1970}. Another popular tensor decomposition technique is the Tucker decomposition~\cite{tucker1966}. These decompositions are essentially generalizations of the matrix SVD and PCA. 
While finding the decomposition of arbitrary tensors is computationally intractable, specially structured tensors appear in some latent variable models. Such tensors can be decomposed efficiently~\cite{anandkumar2012,kolda2009} using a variety of approaches such as generalizations of the power iteration~\cite{lathauwer2000}. Exploiting such structures in higher-order moments to estimate the parameters of latent variable models has been studied extensively using the so-called orthogonal tensor decomposition (OTD)~\cite{anandkumar2012,kakade2013,hsu2012,kolda2015}. To our knowledge, these decompositions have not been studied in the setting of distributed data.

Several distributed PCA algorithms~\cite{balcanOld, balcan2014, borgne2010, bai2005, macua2010,imtiazDPCA2018} have been proposed. Liang et al.~\cite{balcanOld} proposed a distributed PCA scheme where it is necessary to send both the left and right singular vectors along with corresponding singular values from each site to the aggregator. Feldman et al.~\cite{feldman2013} proposed an improvement upon this, where each site sends a $D\times R$ matrix to the aggregator. Balcan et al.~\cite{balcan2014} proposed a further improved version using fast sparse subspace embedding~\cite{clarkson2017} and randomized SVD~\cite{halko2011}. 

This paper proposes new privacy-preserving algorithms for distributed PCA and OTD and builds upon our earlier work on distributed differentially private eigenvector calculations~\cite{imtiazDPCA2018} and centralized differentially private OTD~\cite{imtiazOTD2018}. It improves on our preliminary works on distributed private PCA~\cite{imtiazDJICA2016, imtiazDPCA2018} in terms of efficiency and fault-tolerance. Wang and Anandkumar~\cite{anandkumar2016} recently proposed an algorithm for differentially private tensor decomposition using a noisy version of the tensor power iteration~\cite{anandkumar2012,lathauwer2000}. Their algorithm adds noise at each step of the iteration and the noise variance grows with the predetermined number of iterations. They also make the restrictive assumption that the input to their algorithm is orthogonally decomposable. Our centralized OTD algorithms~\cite{imtiazOTD2018} avoid these assumptions and achieve better empirical performance (although without theoretical guarantees). To our knowledge, this paper proposes the first differentially private orthogonal tensor decomposition algorithm for distributed settings.

%It should be noted that these algorithms do not guarantee any privacy, neither do these algorithms offer efficient latent variable learning in a distributed setting. Wang and Anandkumar~\cite{anandkumar2016} recently proposed an algorithm for differentially private tensor decomposition using a noisy version of the tensor power iteration~\cite{anandkumar2012,lathauwer2000}. Their algorithm adds noise at each step of the iteration and the noise variance grows with the predetermined number of iterations. They also assume that the input to their algorithm is orthogonally decomposable, which is a very contrived assumption. Imtiaz and Sarwate~\cite{imtiazOTD2018} later proposed two improved differentially-private orthogonal tensor decomposition algorithms, which do not have such problems. The authors empirically showed that the two proposed algorithms achieve better performance than the one by Wang and Anandkumar~\cite{anandkumar2016}. However, to the best of our knowledge, the current work is the first foray into distributed differentially-private orthogonal tensor decomposition.

%However, none of these algorithms guarantee any privacy. Imtiaz et al.~\cite{imtiazDJICA2016} proposed a scheme for distributed differentially-private computation of the eigenvectors of the sample second-moment matrix as an intermediate step of a distributed joint independent component analysis algorithm. Very recently, Imtiaz and Sarwate~\cite{imtiazDPCA2018} proposed a distributed differentially-private PCA algorithm that is more fault-tolerant than the previous one.

\noindent\textbf{Our Contribution. }In this paper, we propose two new $(\epsilon,\delta)$-differentially private algorithms, $\capepca$ and $\capeagn$, for distributed differentially private principal component analysis and orthogonal tensor decomposition, respectively. %Each of the algorithms offers $(\epsilon,\delta)$-differential privacy. 
The algorithms are inspired by the recently proposed correlation assisted private estimation ($\cape$) protocol~\cite{jafar2018} and input perturbation methods for differentially-private PCA~\cite{blum2005,dwork2014}.
The $\cape$ protocol improves upon conventional approaches, which suffer from excessive noise, at the expense of requiring a trusted ``helper'' node that can generate correlated noise samples for privacy. We extend the $\cape$ framework to handle site-dependent sample sizes and privacy requirements.
%We illustrate that conventional distributed privacy-preserving algorithms suffer from poor estimation accuracy due to excessive noise. However, in the presence of some additional (reasonable) resources, this sub-optimal performance can be improved for distributed matrix and tensor factorizations using the $\cape$ protocol. We propose an extension to $\cape$ for unequal sample sizes and unequal privacy requirements, which are very useful in practice. Our methods add correlated noise to the output of each site in such a way that the aggregator can combine the outputs ensuring that the noise variance at the aggregator is of the same level as if all data are stored in a centralized location. The proposed methods are provably differentially-private. 
In $\capepca$, the sites share noisy second-moment matrix estimates to a central aggregator, whereas in $\capeagn$ the sites use a distributed protocol to compute a projection subspace used to enable efficient private OTD.
%For the $\capepca$ algorithm, we choose to send the full $D\times D$ noisy second-moment matrix from each site to attain the goal of achieving the same level of utility as the pooled data case. For the proposed $\capeagn$, we compute a projection subspace in the distributed setting that enables us to convert our third-order moment tensor (which is itself computed in a distributed way) into an orthogonally decomposable tensor. 
This paper is about algorithms with provable privacy guarantees and experimental validation. While asymptotic sample complexity guarantees are of theoretical interest, proving performance bounds for distributed subspace estimation is quite challenging. To validate our approach we show that our new methods outperform previously proposed approaches, even under strong privacy constraints. For weaker privacy requirements they can sometimes achieve the same performance as a pooled-data scenario. 

\section{Problems Using Distributed Private Data}\label{sec:problem_formulation}
\begin{figure}[t]
  \centering
  \includegraphics[width=1\columnwidth]{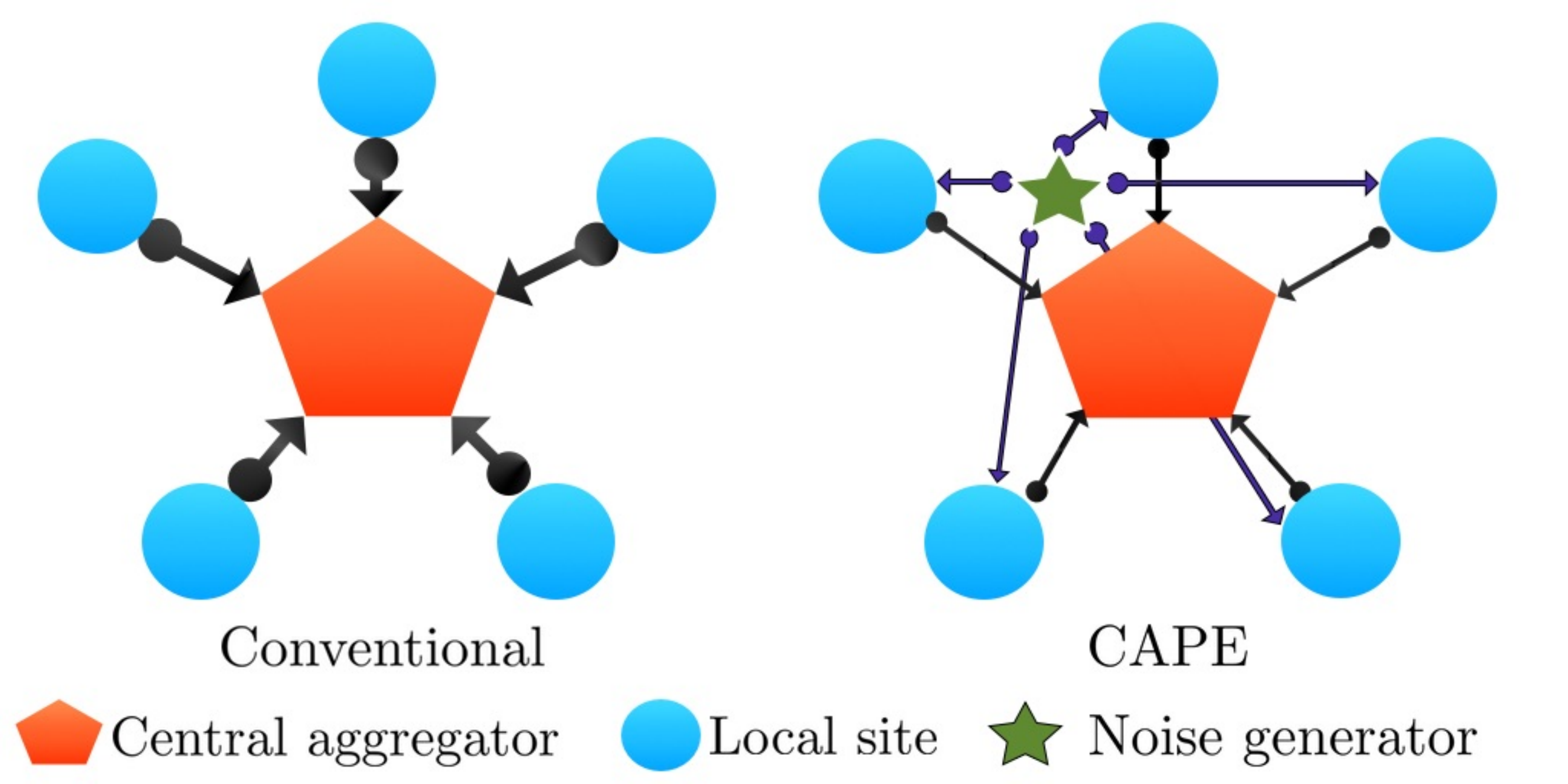}\\
  \vspace{-0.0in}
  \caption{The structure of the network: left -- conventional, right -- $\cape$}
  \label{fig:network_structure}
\end{figure}

\noindent \textbf{Notation.} We denote tensors with calligraphic scripts, e.g., $\tens{X}$, vectors with bold lower case letters, e.g., $\vect{x}$, and matrices with bold upper case letters, e.g. $\matr{X}$. Scalars are denoted with regular letters, e.g., $M$. Indices are denoted with lower case letters and they typically run from 1 to their upper-case versions, e.g., $m = 1, 2, \ldots, M$. We sometimes denote the set $\{1, 2, \ldots, M\}$ as $[M]$. The $n$-th column of the matrix $\matr{X}$ is denoted as $\vect{x}_n$. $\|\cdot\|_2$ denotes the Euclidean (or $\mathcal{L}_2$) norm of a vector and the spectral norm of a matrix. $\|\cdot\|_F$ denotes the Frobenius norm and $\tr(\cdot)$ denotes the trace operation.

\noindent\textbf{Distributed Data Model. }We assume that the data is distributed in $S$ sites, where each site $s \in [S]$ has a data matrix $\matr{X}_s \in \mathbb{R}^{D \times N_s}$. The data samples in the local sites are assumed to be disjoint. There is a central node that acts as an aggregator (see Figure \ref{fig:network_structure}). We denote $N = \sum_{s=1}^S N_s$ as the total number of samples over all sites. The data matrix $\matr{X}_s = \left[\vect{x}_{s,1}\ \ldots\ \vect{x}_{s,N_s}\right]$ at site $s$ is considered to contain the $D$-dimensional features of $N_s$ individuals. Without loss of generality, we assume that $\|\vect{x}_{s,n}\|_2 \leq 1\ \forall s \in [S]$ and $\forall n \in [N_s]$. If we had all the data in the aggregator (pooled data scenario), then the data matrix would be $\matr{X} = \left[\matr{X}_1\ \ldots\ \matr{X}_S\right] \in \mathbb{R}^{D\times N}$. Our goal is to approximate the performance of the pooled data scenario using distributed differentially private algorithms.

\noindent\textbf{Matrix and Tensor Factorizations. }We first formulate the problem of distributed PCA. For simplicity, we assume that the observed samples are mean-centered. The $D\times D$ sample second-moment matrix at site $s$ is $\matr{A}_s = \frac{1}{N_s} \matr{X}_s\matr{X}_s^\top$. In the pooled data scenario, the $D\times D$ positive semi-definite second-moment matrix is $\matr{A} =  \frac{1}{N} \matr{X} \matr{X}^{\top}$. According to the Schmidt approximation theorem \cite{stewart1993}, the rank-$K$ matrix $\matr{A}_K$ that minimizes the difference $\|\matr{A} - \matr{A}_K\|_F$ can be found by taking the SVD of $\matr{A}$ as $\matr{A} = \matr{V} \matr{\Lambda} \matr{V}^\top$, where without loss of generality we assume $\matr{\Lambda}$ is a diagonal matrix with entries $\{\lambda_d(\matr{A})\}$ and $\lambda_1(\matr{A}) \geq \ldots \geq \lambda_D(\matr{A}) \geq 0$. Additionally, $\matr{V}$ is a matrix of eigenvectors corresponding to the eigenvalues. The top-$K$ PCA subspace of $\matr{A}$ is the matrix $\matr{V}_K(\matr{A}) = \left[\vect{v}_1 \ldots \vect{v}_K\right]$. Given $\matr{V}_K(\matr{A})$ and the eigenvalue matrix $\matr{\Lambda}$, we can form an approximation $\matr{A}_K = \matr{V}_K(\matr{A}) \matr{\Lambda}_K \matr{V}_K(\matr{A})^{\top}$ to $\matr{A}$, where $\matr{\Lambda}_K$ contains the $K$ largest eigenvalues in $\matr{\Lambda}$. For a $D \times K$ matrix $\hat{\matr{V}}$ with orthonormal columns, the quality of $\hat{\matr{V}}$ in approximating $\matr{V}_K(\matr{A})$ can be measured by the \textit{captured energy} of $\matr{A}$ as $q(\hat{\matr{V}} ) = \tr(\hat{\matr{V}}^{\top} \matr{A} \hat{\matr{V}})$. The $\hat{\matr{V}}$, which maximizes $q(\hat{\matr{V}})$ is the subspace $\matr{V}_K(\matr{A})$. We are interested in approximating $\matr{V}_K(\matr{A})$ in a distributed setting while guaranteeing differential privacy.

Next, we describe the problem of orthogonal tensor decomposition (OTD). As mentioned before, decomposition of arbitrary tensors is usually mathematically intractable. However, some specially structured tensors that appear in several latent variable models can be efficiently decomposed~\cite{anandkumar2012} using a variety of approaches such as generalizations of the power iteration~\cite{lathauwer2000}. We review some basic definitions related to tensor decomposition~\cite{kolda2009} in Appendix \ref{appendix:tensor_preliminaries}. We start with formulating the problem of orthogonal decomposition of symmetric tensors and then continue on to distributed OTD. Due to page limitations, two examples of OTD from Anandkumar et al.~\cite{anandkumar2012}, namely the single topic model (STM) and the mixture of Gaussian (MOG), are presented in Appendix \ref{appendix:otd_examples}.

Let $\tens{X}$ be an $M$-way $D$ dimensional symmetric tensor. Given real valued vectors $\vect{v}_k \in \mathbb{R}^D$, Comon et al. \cite{comon2008} showed that there exists a decomposition of the form $\tens{X} = \sum_{k=1}^K \lambda_k \vect{v}_k \otimes \vect{v}_k \otimes \cdots \otimes \vect{v}_k$, where $\otimes$ denotes the outer product. Without loss of generality, we can assume that $\|\vect{v}_k\|_2 = 1$ $\forall k$. If we can find a matrix $\matr{V} = \left[\vect{v}_1\ldots \vect{v}_K\right] \in \mathbb{R}^{D \times K}$ with orthogonal columns, then we say that $\tens{X}$ has an orthogonal symmetric tensor decomposition \cite{kolda2015}. Such tensors are generated in several applications involving latent variable models. Recall that if $\matr{M} \in \mathbb{R}^{D \times D}$ is a symmetric rank-$K$ matrix then we know that the SVD of $\matr{M}$ is given by $\matr{M} = \matr{V} \matr{\Lambda} \matr{V}^\top = \sum_{k=1}^K \lambda_k \vect{v}_k \vect{v}_k^\top = \sum_{k=1}^K \lambda_k \vect{v}_k \otimes \vect{v}_k$, where $\matr{\Lambda} = \mbox{diag}\{\lambda_1, \lambda_2, \ldots, \lambda_K\}$ and $\vect{v}_k$ is the $k$-th column of the orthogonal matrix $\matr{V}$. As mentioned before, the orthogonal decomposition of a 3-rd order symmetric tensor $\tens{X} \in \mathbb{R}^{D \times D \times D}$ is a collection of orthonormal vectors $\{\vect{v}_k\}$ together with corresponding positive scalars $\{\lambda_k\}$ such that $\tens{X} = \sum_{k=1}^K \lambda_k \vect{v}_k \otimes \vect{v}_k \otimes \vect{v}_k$. Now, in a setting where the data samples are distributed over different sites, we may have local approximates $\tens{X}_s$. We intend to use these local approximates from all sites to find better and more accurate estimates of the $\{\vect{v}_k\}$, while preserving privacy.

\noindent\textbf{Differential Privacy. }An algorithm $\mathcal{A}(\mathbb{D})$ taking values in a set $\mathbb{T}$ provides $(\epsilon,\delta)$-differential privacy if
\begin{equation}\label{privacy_eqn}
\Pr[\mathcal{A}(\mathbb{D}) \in \mathbb{S}] \leq \exp(\epsilon) \Pr[\mathcal{A}(\mathbb{D'}) \in \mathbb{S}] + \delta,
\end{equation}
for all measurable $\mathbb{S} \subseteq \mathbb{T}$ and all data sets $\mathbb{D}$ and $\mathbb{D'}$ differing in a single entry (neighboring datasets). This definition essentially states that the probability of the output of an algorithm is not changed significantly if the corresponding database input is changed by just one entry. Here, $\epsilon$ and $\delta$ are privacy parameters, where low $\epsilon$ and $\delta$ ensure more privacy. Note that the parameter $\delta$ can be interpreted as the probability that the algorithm fails. For more details, see recent surveys~\cite{SarwateC:13survey} or the monograph of Dwork and Roth~\cite{dwork2013algorithmic}.

To illustrate, consider estimating the mean $f(\vect{x}) = \frac{1}{N} \sum_{n = 1}^N x_n$ of $N$ scalars $\vect{x} = [ x_1,\ldots, x_{N-1},\ x_N]^{\top}$ with each $x_i \in [0,1]$. A neighboring data vector $\vect{x'} = \left[x_1,\ldots, x_{N-1},\ x'_N\right]^\top$ differs in a single element. The sensitivity~\cite{dwork2006} $\max_{\vect{x}} \left|f(\vect{x}) - f(\vect{x'})\right|$ of the function $f(\vect{x})$ is $\frac{1}{N}$. Therefore, for $(\epsilon, \delta)$ differentially-private estimate of the average $a = f(\vect{x})$, we can follow the Gaussian mechanism~\cite{dwork2006} to release $\hat{a} = a + e$, where $e \sim \mathbb{N}\left(0, \tau^2\right)$ and $\tau = \frac{1}{N\epsilon}\sqrt{2\log \frac{1.25}{\delta}}$.

 %To remedy this, the existing approach is to compute the desired function at each site in a privacy-preserving way and send this differentially-private approximate to the aggregator. The aggregator node then computes a consensus of the differentially-private outputs of the sites and releases the final output. In the following, we show an example to demonstrate the problem with this existing approach.

\noindent\textbf{Distributed Privacy-preserving Computation.} In our distributed setting, we assume that the sites are ``honest but curious.'' That is, the aggregator is not trusted and the sites can collude to get a hold of some site's data/function output. Existing approaches to distributed differentially private algorithms can introduce a significant amount of noise to guarantee privacy. Returning to the example of mean estimation, suppose now there are $S$ sites and each site $s$ holds a disjoint dataset $\vect{x}_s$ of $N_s$ samples for $s \in [S]$. A central aggregator wishes to estimate and publish the mean of all the samples. The sites can send estimates to the aggregator but may collude to learn the data of other sites based on the aggregator output. Without privacy, the sites can send $a_s = f(\vect{x}_s)$ to the aggregator and the average computed by aggregator ($a_\mrm{ag} = \frac{1}{S}\sum_{s=1}^S a_s$) is exactly equal to the average we would get if all the data samples were available in the aggregator node. For preserving privacy, a standard differentially private approach is for each site to send $\hat{a}_s = f(\vect{x}_s) + e_s$, where $e_s \sim \mathbb{N}\left(0, \tau_s^2\right)$ and $\tau_s = \frac{1}{N_s\epsilon}\sqrt{2\log \frac{1.25}{\delta}}$. The aggregator computes $a_\mrm{ag} = \frac{1}{S}\sum_{s=1}^S \hat{a}_s$. We observe $a_\mrm{ag} = \frac{1}{S}\sum_{s=1}^S \hat{a}_s = \frac{1}{S}\sum_{s=1}^S a_s + \frac{1}{S}\sum_{s=1}^S e_s$: note that this estimate is still noisy due to the privacy constraint. The variance of the estimator $a_\mrm{ag}$ is $S \cdot \dfrac{\tau_s^2}{S^2} = \dfrac{\tau_s^2}{S} \triangleq \tau_\mrm{ag}^2$. However, if we had all the data samples in the central aggregator, then we could compute the differentially-private average as $a_c = \frac{1}{N}\sum_{n=1}^N x_n + e_c$, where $e_c \sim \mathbb{N}\left(0, \tau_c^2\right)$ and $\tau_c = \frac{1}{N\epsilon}\sqrt{2\log \frac{1.25}{\delta}}$. If we assume that each site has equal number of samples then $N = S N_s$ and we have $\tau_c = \frac{1}{SN_s\epsilon}\sqrt{2\log \frac{1.25}{\delta}} = \dfrac{\tau_s}{S}$. We observe the ratio $\frac{\tau_c^2}{\tau_\mrm{ag}^2} = \frac{\tau_s^2\ /\ S^2}{\tau_s^2\ /\ S} = \frac{1}{S}$, showing that the conventional differentially-private distributed averaging scheme is always worse than the differentially-private pooled data case. 

\section{Correlated Noise Scheme}\label{sec:corr_noise}

The recently proposed Correlation Assisted Private Estimation ($\cape$)~\cite{jafar2018} scheme exploits the network structure and uses a correlated noise design to achieve the same performance of the pooled data case (i.e., $\tau_\mrm{ag} = \tau_c$) in the decentralized setting. We assume there is a trusted noise generator in addition to the central aggregator (see Figure \ref{fig:network_structure}). The local sites and the central aggregator can also generate noise. The noise generator and the aggregator can send noise to the sites through secure (encrypted) channels. The noise addition procedure is carefully designed to ensure the privacy of the algorithm output from each site and to achieve the noise level of the pooled data scenario in the final output from the central aggregator. Considering the same distributed averaging problem as in Section \ref{sec:problem_formulation}, the noise generator and central aggregator respectively send $e_s$ and $f_s$ to each site $s$. Site $s$ generates noise $g_s$ and releases/sends $\hat{a}_s = f(\vect{x}_s) + e_s + f_s + g_s$.
The noise generator generates $e_s$ such that $\sum_{s=1}^S e_s = 0$.  As shown in~\cite{jafar2018}, these noise terms are distributed according to $e_s \sim \mathcal{N}(0, \tau_e^2),\ f_s \sim \mathcal{N}(0, \tau_f^2)$, and $g_s \sim \mathcal{N}(0, \tau_g^2)$, where
\begin{align}\label{eqn:cape_noise_variance}
\tau^2_e &= \tau^2_f = \left(1-\frac{1}{S}\right)\tau^2_s,\mbox{ and } \tau_g^2 = \frac{\tau^2_s}{S}.
\end{align}
The aggregator computes $a_\mrm{ag}^\mrm{imp} = \frac{1}{S}\sum_{s=1}^S \left(\hat{a}_s - f_s\right) = \frac{1}{N} \sum_{n=1}^{N} x_n + \frac{1}{S} \sum_{s=1}^S g_s$, where we used $\sum_s e_s = 0$ and the fact that the aggregator knows the $f_s$, so it can subtract all of those from $\hat{a}_s$. The variance of the estimator $a_\mrm{ag}^\mrm{imp}$ is $S\cdot\frac{\tau_g^2}{S^2} = \frac{\tau_s^2}{S^2} = \tau^2_c$, which is the same as if all the data were present at the aggregator. This claim is formalized in Lemma \ref{lemma:cape}. We show the complete algorithm in Algorithm \ref{alg:dp_avg} (Appendix \ref{appendix:cape:lemma}). Privacy follows from previous work~\cite{jafar2018}, and if $S > 2$ and number of trusted sites (the sites that would not collude with any adversary) $S_\mrm{tr} \geq 2$, the aggregator does not need to generate $f_s$.

%\begin{theorem}[Privacy of $\cape$ Algorithm (Algorithm \ref{alg:dp_avg})~\cite{jafar2018}]\label{thm:cape}
%Algorithm \ref{alg:dp_avg} computes an $(\epsilon, \delta)$-differentially private average $a_\mrm{ag}^\mrm{imp}$.
%\end{theorem}

%%% not needed
%\begin{proof}
%The proof of Theorem \ref{thm:cape} is shown in~\cite{jafar2018}.
%\end{proof}

%%% seems to be not needed?

%\begin{lemma}\label{lemma:cape}
%The distributed averaging following Algorithm \ref{alg:dp_avg} ensures that the additive noise variance at the aggregator is the same as that of the pooled data scenario.
%\end{lemma}
%\begin{proof}
%The proof of Lemma \ref{lemma:cape} is shown in Appendix \ref{appendix:cape:lemma}.
%\end{proof}

%\ads{this seems not needed?}
%
%\begin{prop}\label{prop:cape}
%For $S > 2$ and number of trusted sites $S_\mrm{tr} \geq 2$, the central aggregator does not need to generate the noise $f_s$.
%\end{prop}
%\begin{proof}
%The proof of Proposition \ref{prop:cape} is shown in~\cite{jafar2018}.
%\end{proof}
%
%\ads{Need to define gain, but again I am not sure where this result gets used?}

\begin{prop}\label{prop:gain}(Performance gain~\cite{jafar2018}) Consider the gain function $G(\mathbf{n}) = \frac{\tau_\mrm{ag}^2}{{\tau_\mrm{ag}^\mrm{imp}}^2} = \frac{N^2}{S^2}\sum^S_{s=1}\frac{1}{N_s^2}$ with $\vect{n} = \left[N_1,\ldots, N_S\right]$. Then: 
\begin{itemize}
\item the minimum $G(\mathbf{n})$ is $S$ and is achieved when $\vect{n} = \left[\frac{N}{S}, \ldots, \frac{N}{S}\right]$
\item the maximum $G(\mathbf{n})$ is $\frac{N^2}{S^2}\left(\frac{1}{(N-S+1)^2}+S-1\right)$, which occurs when $\vect{n} = \left[1, \ldots, 1, N-S+1\right]$
\end{itemize}
\end{prop}
\begin{proof}
The proof is a consequence of Schur convexity and is given in~\cite{jafar2018}.
\end{proof}

\subsection{Extension of $\cape$ to Unequal Privacy Requirements}\label{sec:weighted_dp_avg}
We now propose a generalization of the $\cape$ scheme, which applies to scenarios where different sites have different privacy requirements and/or sample sizes. Additionally, sites may have different ``quality notions'', i.e., while combining the site outputs at the aggregator, the aggregator can decide to use different weights to different sites (possibly according to the quality of the output from a site). Let us assume that site $s$ requires $(\epsilon_s, \delta_s)$-differential privacy guarantee for its output. According to the Gaussian mechanism~\cite{dwork2006}, the noise to be added to the (non-private) output of site $s$ should have standard deviation given by $\tau_s = \frac{1}{N_s\epsilon_s}\sqrt{2\log\frac{1.25}{\delta_s}}$. We need that site $s$ outputs $\hat{a}_s = f(\vect{x}_s) + e_s + f_s + g_s$. Here, $g_s \sim \mathcal{N}(0, \tau_{gs}^2)$ is generated locally, $e_s \sim \mathcal{N}(0, \tau_{es}^2)$ is generated from the random noise generator, and $f_s \sim \mathcal{N}(0, \tau_{fs}^2)$ is generated in the central aggregator. We need to satisfy
\begin{align*}
\tau^2_{fs + gs} 	&= \tau_{fs}^2 + \tau_{gs}^2 \geq \tau_s^2,\mbox{ and } \tau^2_{es + gs} = \tau_{es}^2 + \tau_{gs}^2 \geq \tau_s^2.
\end{align*}
As mentioned before, the aggregator can decide to compute a weighted average with weights selected according to some quality measure of the site's data/output (e.g., if the aggregator knows that a particular site is suffering from more noisy observations than other sites, it can choose to give the output from that site less weight while combining the site results). Let us denote the weights by $\{\mu_s\}$ such that $\sum_{s=1}^S \mu_s = 1$ and $\mu_s \geq 0$. Note that, our proposed generalized $\cape$ reduces to the existing~\cite{jafar2018} $\cape$ for $\mu_s = \frac{N_s}{N}$. The aggregator computes
\begin{align*}
a_\mrm{ag}^\mrm{imp} 	&= \sum_{s=1}^S \mu_s \left(\hat{a}_s - f_s\right) = \sum_{s=1}^S \mu_s a_s+ \sum_{s=1}^S \mu_s e_s + \sum_{s=1}^S \mu_s g_s.
\end{align*}
In accordance with our goal of achieving the same level of noise as the pooled data scenario, we need $\text{var}\left[\sum_{s=1}^S \mu_s g_s\right] = \tau_c^2 \implies \sum_{s=1}^S \mu_s^2 \tau_{gs}^2 = \tau_c^2$. Additionally, we need $\sum_{s=1}^S \mu_s e_s = 0$. With these constraints, we can formulate a feasibility problem to solve for the unknown noise variances $\{\tau_{es}^2, \tau_{gs}^2, \tau_{fs}^2\}$ as
\begin{align*}
	\underset{}{\text{minimize}} 	&\qquad 0 \\
				\text{subject to} 	&\qquad \tau_{fs}^2 + \tau_{gs}^2 \geq \tau_s^2,\ \tau_{es}^2 + \tau_{gs}^2 \geq \tau_s^2, \\
                                    &\qquad \sum_{s=1}^S \mu_s^2 \tau_{gs}^2 = \tau_c^2, \sum_{s=1}^S \mu_s e_s = 0,
\end{align*}
for all $s\in [S]$, where $\{\mu_s\}$, $\tau_c$ and $\tau_s$ are known to the aggregator. For this problem, multiple solutions are possible. We present one solution here that solves the problem with equality. For the $S$-th site:
\begin{align*}
\tau_{eS}^2 &= \tau_{fS}^2 = \frac{\tau_S^2}{2} - \frac{1}{2\mu_S^2}\left(\tau_c^2 - \sum_{s=1}^{S-1}\mu_s^2\tau_s^2\right) \\
\tau_{gS}^2 &= \frac{\tau_S^2}{2} + \frac{1}{2\mu_S^2}\left(\tau_c^2 - \sum_{s=1}^{S-1}\mu_s^2\tau_s^2\right).
\end{align*}
For other sites $s \in [S-1]$:
\begin{align*}
\tau_{es}^2 &= \tau_{fs}^2 = \frac{1}{\mu_s^2(S-1)}\left[\frac{\mu_S^2}{2}\tau_S^2 - \frac{1}{2}\left(\tau_c^2 - \sum_{s=1}^{S-1}\mu_s^2\tau_s^2\right)\right] \\
\tau_{gs}^2 &=\tau_s^2 - \frac{1}{\mu_s^2(S-1)}\left[\frac{\mu_S^2}{2}\tau_S^2 - \frac{1}{2}\left(\tau_c^2 - \sum_{s=1}^{S-1}\mu_s^2\tau_s^2\right)\right].
\end{align*}
The derivation of this solution is shown in Appendix \ref{appendix:unequal_privacy}.

\section{Improved Distributed Differentially-private Principal Component Analysis}\label{sec:dist_dppca}
\begin{algorithm}[t] 
	\caption{Improved Distributed Differentially-private PCA ($\capepca$) \label{alg:dist_dpca}}
	\begin{algorithmic}[1]
    \Require Data matrix $\matr{X}_s \in \mathbb{R}^{D\times N_s}$ for $s \in [S]$; privacy parameters $\epsilon$, $\delta$; reduced dimension $K$
    \State At random noise generator: generate $\matr{E}_s \in \mathbb{R}^{D \times D}$, as \hyperlink{target:generate_E}{described} in the text; send to sites
    \State At aggregator: generate $\matr{F}_s \in \mathbb{R}^{D \times D}$, as \hyperlink{target:generate_F}{described} in the text; send to sites
    \For{$s = 1, 2, \ldots, S$} \Comment{at the local sites}
    		\State Compute $\matr{A}_s \gets \frac{1}{N_s} \matr{X}_s \matr{X}_s^\top$
    	\State Generate $D \times D$ symmetric matrix $\matr{G}_s$, as \hyperlink{target:generate_G}{described} in the text
    	\State Compute $\hat{\matr{A}}_s \gets \matr{A}_s + \matr{E}_s + \matr{F}_s + \matr{G}_s$; send $\hat{\matr{A}}_s$ to aggregator
    \EndFor 
    \State Compute $\hat{\matr{A}} \gets \frac{1}{S}\sum_{s=1}^S \left(\hat{\matr{A}}_s - \matr{F}_s\right)$ \Comment{at the aggregator}
    \State Perform SVD: $\hat{\matr{A}} = \matr{V} \matr{\Lambda} \matr{V}^\top$
    \State Release / send to sites: $\matr{V}_K$\\
    \Return $\matr{V}_K$
    \end{algorithmic}
\end{algorithm}
In this section, we propose an improved distributed differentially-private PCA algorithm that takes advantage of the $\cape$ protocol. Recall that in our distributed PCA problem, we are interested in approximating $\matr{V}_K(\matr{A})$ in a distributed setting while guaranteeing differential privacy. One na\"{i}ve approach (non-private) would be to send the data matrices from the sites to the aggregator. When $D$ and/or $N_s$ are large, this entails a huge communication overhead. In many scenarios the local data are also private or sensitive. As the aggregator is not trusted, sending the data to the aggregator can result in a significant privacy violation. Our goals are therefore to reduce the communication cost, ensure differential privacy, and provide a close approximation to the true PCA subspace $\matr{V}_K(\matr{A})$. We previously proposed a differentially-private distributed PCA scheme~\cite{imtiazDPCA2018}, but the performance of the scheme is limited by the larger variance of the additive noise at the local sites due to the smaller sample sizes. We intend to alleviate this problem using the correlated noise scheme~\cite{jafar2018}. The improved distributed differentially-private PCA algorithm $(\capepca)$ we propose here achieves the same utility as the pooled data scenario. 

We consider the same network structure as in Section \ref{sec:corr_noise}: there is a random noise generator that can generate and send noise to the sites through an encrypted/secure channel. The aggregator can also generate noise and send those to the sites over encrypted/secure channels. Recall that in the pooled data scenario, we have the data matrix $\matr{X}$ and the sample second-moment matrix $\matr{A} = \frac{1}{N}  \matr{X}\matr{X}^\top$. We refer to the top-$K$ PCA subspace of this sample second-moment matrix as the true (or optimal) subspace $\matr{V}_K(\matr{A})$. At each site, we compute the sample second-moment matrix as $\matr{A}_s = \frac{1}{N_s} \matr{X}_s \matr{X}_s^\top$. The $\mathcal{L}_2$ sensitivity~\cite{dwork2006} of the function $f(\matr{X}_s) = \matr{A}_s$ is $\Delta_2^s = \frac{1}{N_s}$~\cite{dwork2014}. In order to approximate $\matr{A}_s$ satisfying $(\epsilon, \delta)$ differential privacy, we can employ the $\ag$ algorithm~\cite{dwork2014} to compute $\hat{\matr{A}}_s = \matr{A}_s + \matr{G}_s$, where the symmetric matrix $\matr{G}_s$ is generated with entries i.i.d. $\sim \mathcal{N}(0, \tau_s^2)$ and $\tau_s = \frac{\Delta_2^s}{\epsilon}\sqrt{2\log\frac{1.25}{\delta}}$. Note that, in the pooled data scenario, the $\mathcal{L}_2$ sensitivity of the function $f(\matr{X}) = \matr{A}$ is $\Delta_2^\mrm{pool} = \frac{1}{N}$. Therefore, the required additive noise standard deviation should satisfy $\tau_\mrm{pool} = \frac{\Delta_2^\mrm{pool}}{\epsilon}\sqrt{2\log\frac{1.25}{\delta}} = \frac{\tau_s}{S}$, assuming equal number of samples in the sites. As we want the same utility as the pooled data scenario, we compute the following at each site $s$:
\begin{align*}
\hat{\matr{A}}_s &= \matr{A}_s + \matr{E}_s + \matr{F}_s + \matr{G}_s.
\end{align*}
Here, the noise generator \hypertarget{target:generate_E}{generates} the $D\times D$ matrix $\matr{E}_s$ with $[\matr{E}_s]_{ij}$ drawn i.i.d. $\sim \mathcal{N}(0,\tau_e^2)$ and $\sum_{s=1}^S \matr{E}_s = 0$. We set the variance $\tau_e^2$ according to \eqref{eqn:cape_noise_variance} as $\tau_e^2 = \left(1-\frac{1}{S}\right) \tau_s^2$. Additionally, the aggregator \hypertarget{target:generate_F}{generates} the $D\times D$ matrix $\matr{F}_s$ with $[\matr{F}_s]_{ij}$ drawn i.i.d. $\sim \mathcal{N}(0,\tau_f^2)$. The variance $\tau_f^2$ is set according to \eqref{eqn:cape_noise_variance} as $\tau_f^2 = \left(1-\frac{1}{S}\right) \tau_s^2$. Finally, the sites \hypertarget{target:generate_G}{generate} their own symmetric $D \times D$ matrix $\matr{G}_s$, where $[\matr{G}_s]_{ij}$ are drawn i.i.d. $\sim \mathcal{N}(0, \tau_g^2)$ and $\tau_g^2 = \frac{1}{S} \tau_s^2$ according to \eqref{eqn:cape_noise_variance}. Note that, these variance assignments can be readily modified to fit the unequal  privacy/sample size scenario (Section \ref{sec:weighted_dp_avg}). However, for simplicity, we are considering the equal sample size scenario. Now, the sites send their $\hat{\matr{A}}_s$ to the aggregator and the aggregator \hypertarget{target:A_hat}{computes} 
\begin{align*}
\hat{\matr{A}} &= \frac{1}{S} \sum_{s=1}^S \left(\hat{\matr{A}}_s - \matr{F}_s\right) = \frac{1}{S} \sum_{s=1}^S \left(\matr{A}_s + \matr{G}_s\right),
\end{align*}
where we used the relation $\sum_{s=1}^S \matr{E}_s = 0$. The detailed calculation is shown in Appendix \ref{appendix:A_hat}. We note that at the aggregator, we end up with an estimator with noise variance exactly the same as that of the pooled data scenario. Next, we perform SVD on $\hat{\matr{A}}$ and release the top-$K$ eigenvector matrix $\matr{V}_K$, which is the $(\epsilon, \delta)$ differentially private approximate to the true subspace $\matr{V}_K(\matr{A})$. To achieve the same utility level as the pooled data case, we chose to send the full matrix $\hat{\matr{A}}_s$ from the sites to the aggregator instead of the partial square root of it~\cite{imtiazDPCA2018}. This increases the communication cost by $SD(D-R)$, where $R$ is the intermediate dimension of the partial square root. This can be thought of as the cost of gain in performance. 

\begin{theorem}[Privacy of $\capepca$ Algorithm]\label{thm:dp_dpca}
Algorithm \ref{alg:dist_dpca} computes an $(\epsilon,\delta)$ differentially private approximation to the optimal subspace $\matr{V}_K(\matr{A})$.
\end{theorem}

\begin{proof}[Proof] The proof of Theorem \ref{thm:dp_dpca} follows from using the Gaussian mechanism~\cite{dwork2006}, the AG algorithm~\cite{dwork2014}, the bound on $\|\matr{A}_s - \matr{A'}_s\|_2$ and recalling that the data samples in each site are disjoint. We start by showing that 
\begin{align*}
\tau_e^2 + \tau_g^2 &= \tau_g^2 + \tau_f^2 = \tau_s^2 = \left(\frac{1}{N_s\epsilon}\sqrt{2\log\frac{1.25}{\delta}}\right)^2.
\end{align*}
Therefore, the computation of $\hat{\matr{A}}_s$ at each site is at least $(\epsilon, \delta)$ differentially-private. As differential privacy is invariant under post-processing, we can combine the noisy second-moment matrices $\hat{\matr{A}}_s$ at the aggregator while subtracting $\matr{F}_s$ for each site. By the correlated noise generation at the random noise generator, the noise $\matr{E}_s$ cancels out. We perform the SVD on $\hat{\matr{A}}$ and release $\matr{V}_K$. The released subspace $\matr{V}_K$ is thus the $(\epsilon, \delta)$ differentially private approximate to the true subspace $\matr{V}_K(\matr{A})$.
\end{proof}

\noindent\textbf{Performance Gain with Correlated Noise. }\hypertarget{target:dpca_perf_gain}{The} distributed differentially-private PCA algorithm of~\cite{imtiazDPCA2018} essentially employs the conventional averaging (when each site sends the full $\hat{\matr{A}}_s$ to the aggregator). Therefore, the gain in performance of the proposed $\capepca$ algorithm over the one in~\cite{imtiazDPCA2018} is the same as shown in Proposition \ref{prop:gain}.

\noindent\textbf{Theoretical Performance Guarantee. }Due to the application of the correlated noise protocol, we achieve the same level of noise at the aggregator in the distributed setting as we would have in the pooled data scenario. In essence, the proposed $\capepca$ algorithm can achieve the same performance as the $\ag$ algorithm~\cite{dwork2014} modified to account for all the samples across all the sites. Here, we present three guarantees for the captured energy, closeness to the true subspace and low-rank approximation. The guarantees are adopted from Dwork et al.~\cite{dwork2014} and modified to fit our setup and notation. Let us assume that the $(\epsilon, \delta)$ differentially-private subspace output from Algorithm \ref{alg:dist_dpca} and the true subspace are denoted by $\hat{\matr{V}}_K$ and $\matr{V}_K$, respectively. We denote the singular values of $\matr{X}$ with $\sigma_1 \geq \ldots\geq\sigma_D$ and the un-normalized second-moment matrix with $\matr{A} = \matr{X}\matr{X}^\top$. Let $\matr{A}_K$ and $\hat{\matr{A}}_K$ be the true and the $(\epsilon, \delta)$ differentially-private rank-$K$ approximates to $\matr{A}$, respectively. If we assume that the gap $\sigma_K^2 - \sigma_{K+1}^2 = \omega(\tau_\mrm{pool}\sqrt{D})$, then the following holds
\begin{itemize}
\item $\tr\left(\hat{\matr{V}}_K^\top\matr{A}\hat{\matr{V}}_K\right) \geq \tr\left(\matr{V}_K^\top\matr{A}\matr{V}_K\right) - O(\tau_\mrm{pool}K\sqrt{D})$
\item $\left\|\matr{V}_K\matr{V}_K^\top - \hat{\matr{V}}_K\hat{\matr{V}}_K^\top\right\|_2 = O\left(\frac{\tau_\mrm{pool}\sqrt{D}}{\sigma_K^2 - \sigma_{K+1}^2}\right)$
\item $\|\matr{A} - \hat{\matr{A}}_K\|_2 \leq \|\matr{A} - \matr{A}_K\|_2 + O(\tau_\mrm{pool}\sqrt{D})$.
\end{itemize}
The detailed proofs can be found in Dwork et al.~\cite{dwork2014}.

\noindent\textbf{Communication Cost. }We quantify the total communication cost associated with the proposed $\capepca$ algorithm. Recall that $\capepca$ is an one-shot algorithm. Each of the random noise generator and the aggregator send one $D\times D$ matrix to the sites. Each site uses these to compute the noisy estimate of the local second-moment matrix ($D\times D$) and sends that back to the aggregator. Therefore, the total communication cost is proportional to $3SD^2$ or $O(D^2)$. This is expected as we are computing the global $D\times D$ second-moment matrix in a distributed setting before computing the PCA subspace.

\section{Distributed Differentially-private Orthogonal Tensor Decomposition}\label{sec:dist_dpotd}
\begin{algorithm}[t] 
	\caption{Distributed Differentially-private OTD ($\capeagn$) \label{alg:distagn}}
	\begin{algorithmic}[1]
    \Require Sample second-order moment matrices $\matr{M}_2^s \in \mathbb{R}^{D\times D}$ and third-order moment tensors $\tens{M}_3^s \in \mathbb{R}^{D \times D \times D}$ $\forall s \in [S]$, privacy parameters $\epsilon_1$, $\epsilon_2$, $\delta_1$, $\delta_2$, reduced dimension $K$
    \State At random noise generator: generate $\matr{E}_2^s \in \mathbb{R}^{D \times D}$ and $\tens{E}_3^s \in \mathbb{R}^{D \times D \times D}$, as \hyperlink{target:generate_E2}{described} in the text; send to sites
		    
    \State At aggregator: generate $\matr{F}_2^s \in \mathbb{R}^{D \times D}$ and $\tens{F}_3^s \in \mathbb{R}^{D \times D \times D}$, as \hyperlink{target:generate_F2}{described} in the text; send to sites
	    
    \For{$s = 1,\ \ldots,\ S$}\Comment{at the local sites}
        \State Generate $\matr{G}_2^s \in \mathbb{R}^{D\times D}$, as \hyperlink{target:generate_G2}{described} in the text
        \State Compute $\hat{\matr{M}}_2^s \gets \matr{M}_2^s + \matr{E}_2^s + \matr{F}_2^s + \matr{G}_2^s$; send $\hat{\matr{M}}_2^s$ to aggregator
    \EndFor
    \State Compute $\hat{\matr{M}}_2 \gets \frac{1}{S}\sum_{s=1}^S \left(\hat{\matr{M}}_2^s - \matr{F}_2^s\right)$ and then SVD$(K)$ of $\hat{\matr{M}}_2$ as $\hat{\matr{M}}_2 = \matr{U}\matr{D}\matr{U}^\top$\Comment{at the aggregator}
    \State Compute and send to sites: $\matr{W} \gets \matr{U}\matr{D}^{-\frac{1}{2}}$
    \For{$s = 1,\ \ldots,\ S$}\Comment{at the local sites}
%        	\State Draw a random vector $\vect{b} \in \mathbb{R}^{D_\mrm{sym}}$ with entries i.i.d. $\sim \mathcal{N}(0, \tau_{3g}^2)$, where $\tau_{3g}^2 = \frac{1}{S}{\tau_3^s}^2$
        	\State Generate symmetric $\tens{G}_3^s \in \mathbb{R}^{D\times D \times D}$ from the entries of $\vect{b}\in \mathbb{R}^{D_\mrm{sym}}$, where $[\vect{b}]_d \sim \mathcal{N}(0, \tau_{3g}^2)$ and $\tau_{3g}^2 = \frac{1}{S}{\tau_3^s}^2$
        	\State Compute $\hat{\tens{M}}_3^s \gets \tens{M}_3^s + \tens{E}_3^s + \tens{F}_3^s + \tens{G}_3^s$ and $\tilde{\tens{M}}_3^s \gets \hat{\tens{M}}_3^s\left(\matr{W}, \matr{W}, \matr{W}\right)$; send $\tilde{\tens{M}}_3^s$ to aggregator
    \EndFor
    \State Compute $\tilde{\tens{M}}_3 \gets \frac{1}{S}\sum_{s=1}^S \left(\tilde{\tens{M}}_3^s - \tens{F}_3^s\left(\matr{W}, \matr{W}, \matr{W}\right)\right)$\Comment{at the aggregator}\\
    \Return The differentially private orthogonally decomposable tensor $\tilde{\tens{M}}_{3} $, projection subspace $\matr{W}$
    \end{algorithmic}
\end{algorithm}

In this section, we propose an algorithm $(\capeagn)$ for distributed differentially-private OTD. The proposed algorithm takes advantage of the correlated noise design scheme (Algorithm \ref{alg:dp_avg})~\cite{jafar2018}. To our knowledge, this is the first work on distributed differentially-private OTD. Due to page limits, the definition of the differentially-private OTD and the description of two recently proposed differentially-private OTD algorithms~\cite{imtiazOTD2018} are presented in Appendix \ref{appendix:dpotd}.

We start with recalling that the orthogonal decomposition of a 3-rd order symmetric tensor $\tens{X} \in \mathbb{R}^{D \times D \times D}$ is a collection of orthonormal vectors $\{\vect{v}_k\}$ together with corresponding positive scalars $\{\lambda_k\}$ such that $\tens{X} = \sum_{k=1}^K \lambda_k \vect{v}_k \otimes \vect{v}_k \otimes \vect{v}_k$. A unit vector $\vect{u} \in \mathbb{R}^D$ is an \textbf{eigenvector} of $\tens{X}$ with corresponding \textbf{eigenvalue} $\lambda$ if $\tens{X}(\matr{I},\vect{u},\vect{u}) = \lambda \vect{u}$, where $\matr{I}$ is the $D \times D$ identity matrix~\cite{anandkumar2012}. To see this, one can observe
\begin{align*}
\tens{X}(\matr{I},\vect{u},\vect{u})
					&= \sum_{k=1}^K \lambda_k \left(\matr{I}^\top \vect{v}_k\right) \otimes 									\left(\vect{u}^\top \vect{v}_k\right) \otimes \left(\vect{u}^\top \vect{v}_k\right) \\
                    &= \sum_{k=1}^K \lambda_k \left(\vect{u}^\top \vect{v}_k\right)^2 \vect{v}_k.
\end{align*}
By the orthogonality of the $\vect{v}_k$, it is clear that $\tens{X}(\matr{I},\vect{v}_k,\vect{v}_k) =  \lambda_k \vect{v}_k$ $\forall k$. Now, the orthogonal tensor decomposition proposed in~\cite{anandkumar2012} is based on the mapping
\begin{align} \label{eqn_tensor_power_method}
\vect{u} \mapsto \frac{\tens{X}(\matr{I},\vect{u},\vect{u})}{\|\tens{X}(\matr{I},\vect{u},\vect{u})\|_2},
\end{align}
which can be considered as the tensor equivalent of the well-known matrix power method. Obviously, all tensors are not orthogonally decomposable. As the tensor power method requires the eigenvectors $\{\vect{v}_k\}$ to be orthonormal, we need to perform \emph{whitening} - that is, projecting the tensor on a subspace such that the eigenvectors become orthogonal to each other. 

We note that the proposed algorithm applies to both of the STM and MOG problems. However, as the correlated noise scheme only works with Gaussian noise, the proposed $\capeagn$ employs the $\agn$ algorithm~\cite{imtiazOTD2018} at its core. 
%We recall our distributed data setting: we have $S$ number of sites, each with $N_s$ number of samples.
In-line with our setup in Section \ref{sec:corr_noise}, we assume that there is a random noise generator that can generate and send noise to the sites through an encrypted/secure channel. The un-trusted aggregator can also generate noise and send those to the sites over encrypted/secure channels. At site $s$, the sample second-order moment matrix and the third-order moment tensor are denoted as $\matr{M}_2^s \in \mathbb{R}^{D\times D}$ and $\tens{M}_3^s \in \mathbb{R}^{D\times D\times D}$, respectively. The noise standard deviation required for computing the $(\epsilon_1, \delta_1)$ differentially-private approximate to $\matr{M}_2^s$ is given by
\begin{align}\label{eqn:tau2s}
\tau_2^s &= \frac{\Delta_2^s}{\epsilon_1} \sqrt{2\log\left(\frac{1.25}{\delta_1}\right)},
\end{align}
where the sensitivity $\Delta_2^s$ is inversely \hypertarget{orig:sensitivity_M2}{proportiona}l to the sample size $N_s$. To be more specific, we can write $\Delta_{2,S}^s = \frac{\sqrt{2}}{N_s}$ and $\Delta_{2,M}^s = \frac{1}{N_s}$. The detailed derivation of the sensitivity of $\matr{M}_2^s$ for both STM and MOG are \hyperlink{target:sensitivity_M2}{shown} in Appendix \ref{appendix:dpotd}. Additionally, at site $s$, the noise standard deviation required for computing the $(\epsilon_2, \delta_2)$ differentially-private approximate to $\tens{M}_3^s$ is given by
	\begin{align}\label{eqn:tau3s}
	\tau_3^s &= \frac{\Delta_3^s}{\epsilon_2} \sqrt{2\log\left(\frac{1.25}{\delta_2}\right)}.
	\end{align}
Again, we can write $\Delta_{3,S}^s = \frac{\sqrt{2}}{N_s}$ and $\Delta_{3,M}^s = \frac{2}{N_s} + \frac{6D\sigma^2}{N_s}$. Appendix \ref{appendix:dpotd} contains the  detailed algebra for calculating the sensitivity of $\tens{M}_3^s$ for STM and MOG. 
%The detailed algebra for calculating the sensitivity of $\tens{M}_3^s$ for both STM and MOG are \hyperlink{target:sensitivity_M3}{shown} in Appendix \ref{appendix:dpotd}. 
We note that, as in the case of $\matr{M}_2^s$, the sensitivity \hypertarget{orig:sensitivity_M3}{depends} only on the sample size $N_s$. 
Now, in the pooled-data scenario, the noise standard deviations would be given by:
	\begin{align*}
	\tau_2^\mrm{pool} &= \frac{\Delta_2^\mrm{pool}}{\epsilon_1} 
		\sqrt{2\log\left(\frac{1.25}{\delta_1}\right)}\\
	\tau_3^\mrm{pool} &= \frac{\Delta_3^\mrm{pool}}{\epsilon_2} 
		\sqrt{2\log\left(\frac{1.25}{\delta_2}\right)},
	\end{align*}
where $\Delta_2^\mrm{pool} = \frac{\Delta_2^s}{S}$ and $\Delta_3^\mrm{pool} = \frac{\Delta_3^s}{S}$, assuming equal number of samples in the sites.
% In order for a differentially-private distributed OTD algorithm, 
We need to compute the $D\times K$ whitening matrix $\matr{W}$ and the $D\times D\times D$ tensor $\hat{\tens{M}}_3$ in a distributed way while satisfying differential privacy. Although we could employ our previous centralized differentially-private distributed PCA algorithm~\cite{imtiazDPCA2018} to compute $\matr{W}$, to achieve the same level of accuracy as the pooled data scenario, we compute the following matrix at site $s$:
	\begin{align*}
	\hat{\matr{M}}_2^s &= \matr{M}_2^s + \matr{E}_2^s + \matr{F}_2^s + \matr{G}_2^s,
	\end{align*}
where $\matr{E}_2^s \in \mathbb{R}^{D\times D}$ is \hypertarget{target:generate_E2}{generated} at the noise generator satisfying $\sum_{s=1}^S \matr{E}_2^s = 0$ and the entries $\left[\matr{E}_2^s\right]_{ij}$ drawn i.i.d. $\sim \mathcal{N}(0, \tau_{2e}^2)$. Here, we set the noise variance according to \eqref{eqn:cape_noise_variance}: $\tau_{2e}^2 = \left(1-\frac{1}{S}\right) {\tau_2^s}^2$. Additionally, $\matr{F}_2^s \in \mathbb{R}^{D\times D}$ is \hypertarget{target:generate_F2}{generated} at the aggregator with the entries $\left[\matr{F}_2^s\right]_{ij}$ drawn i.i.d. $\sim \mathcal{N}(0, \tau_{2f}^2)$. We set the noise variance according to \eqref{eqn:cape_noise_variance}: $\tau_{2f}^2 = \left(1-\frac{1}{S}\right) {\tau_2^s}^2$. Finally, $\matr{G}_2^s \in \mathbb{R}^{D\times D}$ is a symmetric matrix \hypertarget{target:generate_G2}{generated} at site $s$ where $\{\left[\matr{G}_2^s\right]_{ij}: i \in [D], j \leq i\}$ are drawn i.i.d. $\sim \mathcal{N}(0, \tau_{2g}^2)$, $[\matr{G}_2^s]_{ij} = [\matr{G}_2^s]_{ji}$ and $\tau_{2g}^2 = \frac{1}{S} {\tau_2^s}^2$. At the aggregator, we compute
	\begin{align*}
	\hat{\matr{M}}_2 &= \frac{1}{S}\sum_{s=1}^S \left(\hat{\matr{M}}_2^s - 	\matr{F}_2^s\right) = \frac{1}{S}\sum_{s=1}^S \left(\matr{M}_2^s + \matr{G}_2^s\right),
\end{align*}
where we used the relation $\sum_{s=1}^S \matr{E}_2^s = 0$. Note that the variance of the additive noise in $\hat{\matr{M}}_2$ is exactly the same as the pooled data scenario, as described in Section \ref{sec:corr_noise}. At the aggregator, we can then compute the SVD($K$) of $\hat{\matr{M}}_2$ as $\hat{\matr{M}}_2 = \matr{U} \matr{D} \matr{U}^\top$. We compute the matrix $\matr{W} = \matr{U}\matr{D}^{-\frac{1}{2}}$ and send it to the sites. 

Next, we focus on computing $\hat{\tens{M}}_3$ in the distributed setting. For this purpose, we can follow the same steps as computing $\hat{\matr{M}}_2$. However, $\hat{\tens{M}}_3$ is a $D\times D\times D$ tensor, and for large enough $D$, this will entail a very large communication overhead. We alleviate this in the following way: each site receives $\tens{F}_3^s \in \mathbb{R}^{D\times D\times D}$ and $\matr{W}$ from the aggregator and $\tens{E}_3^s \in \mathbb{R}^{D\times D\times D}$ from the noise generator. Here, $[\tens{F}_3^s]_{ijk}$ are drawn i.i.d. $\sim \mathcal{N}(0, \tau_{3f}^2)$. Additionally, $[\tens{E}_3^s]_{ijk}$ are drawn i.i.d. $\sim \mathcal{N}(0, \tau_{3e}^2)$ and $\sum_{s=1}^S \tens{E}_3^s = 0$ is satisfied. We set the two variance terms according to \eqref{eqn:cape_noise_variance}: $\tau_{3f}^2 = \tau_{3e}^s = \left(1-\frac{1}{S}\right) {\tau_3^s}^2$. Finally, each site generates their own $\tens{G}_3^s \in \mathbb{R}^{D\times D\times D}$ in the following way: site $s$ draws a vector $\vect{b} \in \mathbb{R}^{D_\mrm{sym}}$ with $D_\mrm{sym}={D+2 \choose 3}$ and entries i.i.d. $\sim \mathcal{N}(0, \tau_{3g}^2)$, where $\tau_{3g}^2 = \frac{1}{S}{\tau_3^s}^2$. The tensor $\tens{G}_3^s$ is generated with the entries from $\vect{b}$ such that $\tens{G}_3^s$ is symmetric. Again, for both $\hat{\matr{M}}_2^s$ and $\hat{\tens{M}}_3^s$, we are considering the equal sample size scenario for simplicity. Our framework requires only a small modification to incorporate the unequal privacy/sample size scenario (Section \ref{sec:weighted_dp_avg}). Now, each site $s$ computes
	\begin{align*}
	\hat{\tens{M}}_3^s 
	&= \tens{M}_3^s + \tens{E}_3^s + \tens{F}_3^s + \tens{G}_3^s \mbox{ and } \tilde{\tens{M}}_3^s = \hat{\tens{M}}_3^s\left(\matr{W}, \matr{W}, \matr{W}\right).
	\end{align*}
We note that $\tilde{\tens{M}}_3^s$ is a $K\times K\times K$ dimensional tensor. Each site sends this to the aggregator. This saves a lot of communication overhead as typically $K \ll D$. To see how this would result in the same estimate of $\tilde{\tens{M}}_3$ as the pooled data scenario, we \hypertarget{target:M3_tilde}{observe}
\begin{align*}
\tilde{\tens{M}}_3^s &=  \hat{\tens{M}}_3^s\left(\matr{W}, \matr{W}, \matr{W}\right) = \tens{M}_3^s\left(\matr{W}, \matr{W}, \matr{W}\right) + \\
& \tens{E}_3^s\left(\matr{W}, \matr{W}, \matr{W}\right) + \tens{F}_3^s\left(\matr{W}, \matr{W}, \matr{W}\right) + \tens{G}_3^s\left(\matr{W}, \matr{W}, \matr{W}\right).
\end{align*}
Additionally, at the aggregator, we compute 
\begin{align*}
\tilde{\tens{M}}_3 	&= \frac{1}{S}\sum_{s=1}^S \left(\tilde{\tens{M}}_3^s - \tilde{\tens{F}}_3^s\right) \\
							&= \left(\frac{1}{S}\sum_{s=1}^S \tens{M}_3^s + \tens{G}_3^s\right)\left(\matr{W}, \matr{W}, \matr{W}\right),
\end{align*}
where $\tilde{\tens{F}}_3^s = \tens{F}_3^s\left(\matr{W}, \matr{W}, \matr{W}\right)$. We used the associativity of the multi-linear operation~\cite{anandkumar2012} and the relation $\sum_{s=1}^S \tens{E}_3^s = 0$. The detailed calculation is \hyperlink{target:M3_tilde_calculation}{shown} in Appendix \ref{appendix:M3_tilde}. Note that the $\tilde{\tens{M}}_3$ we achieve in this scheme is exactly the same $\tilde{\tens{M}}_3$ we would have achieved if all the data samples were present in the aggregator. Moreover, this is also the quantity that the aggregator would get if the sites send the full $\hat{\tens{M}}_3^s$ to the aggregator instead of $\tilde{\tens{M}}_3^s$. The complete $\capeagn$ algorithm is shown in Algorithm \ref{alg:distagn}.

\begin{theorem}[Privacy of $\capeagn$ Algorithm]\label{thm:distagn}
Algorithm \ref{alg:distagn} computes an $(\epsilon_1+\epsilon_2,\delta_1+\delta_2)$ differentially private orthogonally decomposable tensor $\tilde{\tens{M}}_3$. Additionally, the computation of the projection subspace $\matr{W}$ is $(\epsilon_1,\delta_1)$ differentially private. 
\end{theorem}

\begin{proof}[Proof] The proof of Theorem \ref{thm:distagn} follows from using the Gaussian mechanism~\cite{dwork2006}, the sensitivities of $\matr{M}_2^s$ and $\tens{M}_3^s$ and recalling that the data samples in each site are disjoint. First, we show that the computation of $\matr{W}$ satisfies $(\epsilon_1,\delta_1)$ differential privacy. Due to the nature of the correlated noise design, we have
\begin{align*}
\tau^2_{2e} + \tau^2_{2g} &= \tau^2_{2g} + \tau^2_{2f} = {\tau_2^s}^2 = \left(\frac{\Delta_2^s}{\epsilon_1} \sqrt{2\log\left(\frac{1.25}{\delta_1}\right)}\right)^2,
\end{align*}
where $\Delta_2^s$ is the sensitivity of $\matr{M}_2^s$. Therefore, the release of $\hat{\matr{M}}_2^s$ from each site $s$ is at least $(\epsilon_1, \delta_1)$ differentially-private. As differential privacy is closed under post-processing and the samples in each site are disjoint, the computation of $\matr{W}$ at the aggregator also satisfies $(\epsilon_1, \delta_1)$ differential privacy. Next, we show that the computation of $\tilde{\tens{M}}_3$ satisfies $(\epsilon_1+\epsilon_2,\delta_1+\delta_2)$ differential privacy. We recall that
\begin{align*}
\tau^2_{3e} + \tau^2_{3g} &= \tau^2_{3g} + \tau^2_{3f} = {\tau_3^s}^2 = \left(\frac{\Delta_3^s}{\epsilon_2} \sqrt{2\log\left(\frac{1.25}{\delta_2}\right)}\right)^2,
\end{align*}
where $\Delta_3^s$ is the sensitivity of $\tens{M}_3^s$. The computation of $\hat{\tens{M}}_3^s$ at each site is at least $(\epsilon_2, \delta_2)$ differentially-private. Further, by the composition theorem~\cite{dwork2006}, the computation $\tilde{\tens{M}}_3^s = \hat{\tens{M}}_3^s\left(\matr{W}, \matr{W}, \matr{W}\right)$ at each site is $(\epsilon_1+\epsilon_2, \delta_1+\delta_2)$ differentially-private. By the post-processing invariability, the computation of $\tilde{\tens{M}}_3$ at the aggregator is $(\epsilon_1+\epsilon_2, \delta_1+\delta_2)$ differentially-private.
\end{proof}

\noindent\textbf{Performance Gain with Correlated Noise. }As we mentioned before, this is the first work that proposes an algorithm for distributed differentially-private OTD. As we employ the $\cape$ scheme for our computations, the gain in the performance over a conventional distributed differentially-private OTD is therefore the same as in the case of distributed differentially-private averaging, as described in Proposition \ref{prop:gain}.

\noindent\textbf{Theoretical Performance Guarantee. }Although our proposed $\capeagn$ algorithm can reach the performance of the pooled data scenario (that is, the $\agn$ algorithm with all data samples from all sites stored in the aggregator), it is hard to characterize how the estimated $\{\hat{\vect{a}}_k\}$ and $\{\hat{w}_k\}$ would deviate from the true $\{\vect{a}_k\}$ and $\{w_k\}$, respectively. We note that although we are adding symmetric noise to the third-order moment tensor, an orthogonal decomposition need not exist for the perturbed tensor, even though the perturbed tensor is symmetric~\cite{anandkumar2012,kolda2015}. Anandkumar et al.~\cite{anandkumar2012} provided a bound on the error of the recovered decomposition in terms of the operator norm of the tensor perturbation. For our proposed algorithm, the perturbation includes the effect of estimating the third-order moment tensor from the samples as well as the noise added for differential-privacy. Even without accounting for the error in estimating the moments from observable samples, the operator norm of the effective noise at the aggregator: $\|\tens{G}\|_\mrm{op} = \frac{1}{S}\left\|\sum_{s=1}^S\tens{G}_3^s\right\|_\mrm{op}$, is a random quantity, and requires new measure concentration results to analyze. Relating these bounds to the error in estimating recovering the $\{\vect{a}_k\}$ and $\{w_k\}$ is nontrivial and we defer this for future work.

\noindent\textbf{Communication Cost.} We note that $\capeagn$ is a two-step algorithm: it 
computes the projection subspace $\matr{W} \in \mathbb{R}^{D \times K}$ and then orthogonally decomposable tensor $\tilde{\tens{M}}_3$. The random noise generator sends $\matr{E}_2^s \in \mathbb{R}^{D \times D}$ and $\tens{E}_3^s \in \mathbb{R}^{D \times D \times D}$ to each site $s$. Each site $s$ sends $\hat{\matr{M}}_2^s \in \mathbb{R}^{D\times D}$ and $\tilde{\tens{M}}_3^s \in \mathbb{R}^{K \times K \times K}$ and  to the aggregator. The aggregator sends $\matr{F}_2^s \in \mathbb{R}^{D \times D}$, $\matr{W} \in \mathbb{R}^{D \times K}$,  and $\tens{F}_3^s \in \mathbb{R}^{D \times D \times D}$ to each site $s$.
%We quantify the total communication cost associated with the proposed $\capeagn$ algorithm. We note that $\capeagn$ is a two-step algorithm. The first step is to compute the projection subspace $\matr{W} \in \mathbb{R}^{D \times K}$ in the distributed setting and the second step is to compute the orthogonally decomposable tensor $\tilde{\tens{M}}_3$ in the distributed setting. Each of the random noise generator and the aggregator send one $D\times D$ matrix and one $D\times D\times D$ tensor to the sites. The sites first send the $D\times D$ noisy second-moment matrix to the aggregator. The aggregator then computes the estimate of $\matr{W}$ and sends that back to the sites. Each site then computes the local noisy estimate of the third-order tensor and the projection of the said tensor on the subspace. The $K\times K \times K$ projected tensor is then sent to the aggregator from each site. 
Therefore, the total communication cost is proportional to $3SD^2 + 2SD^3 + SDK + SK^3$ or $O(D^3)$. This is expected as we are computing the global $D\times D \times D$ third-order moment tensor in a distributed setting.

\section{Experimental Results}\label{sec:experimental_results}

In this section, we empirically show the effectiveness of the proposed distributed differentially-private matrix and tensor factorization algorithms.
 %As evident from the discussions of the said algorithms, the parameter space is huge. Therefore, in general, we intend to 
We focus on investigating the privacy-utility trade-off:
 %. We are also interested to know 
how the performance varies as a function of the privacy parameters and the number of samples. We start with the %experimental results of 
the proposed $\capepca$ algorithm
%. Then we proceed to showing the simulation results of the 
followed by the
$\capeagn$ algorithm. In each case, we compare the proposed algorithms with existing (if any) algorithm, non-private algorithm and a conventional approach (no correlated noise).

\subsection{Improved Distributed Differentially-private PCA}
\begin{figure*}[t]
  \centering
  \includegraphics[width=1\textwidth]{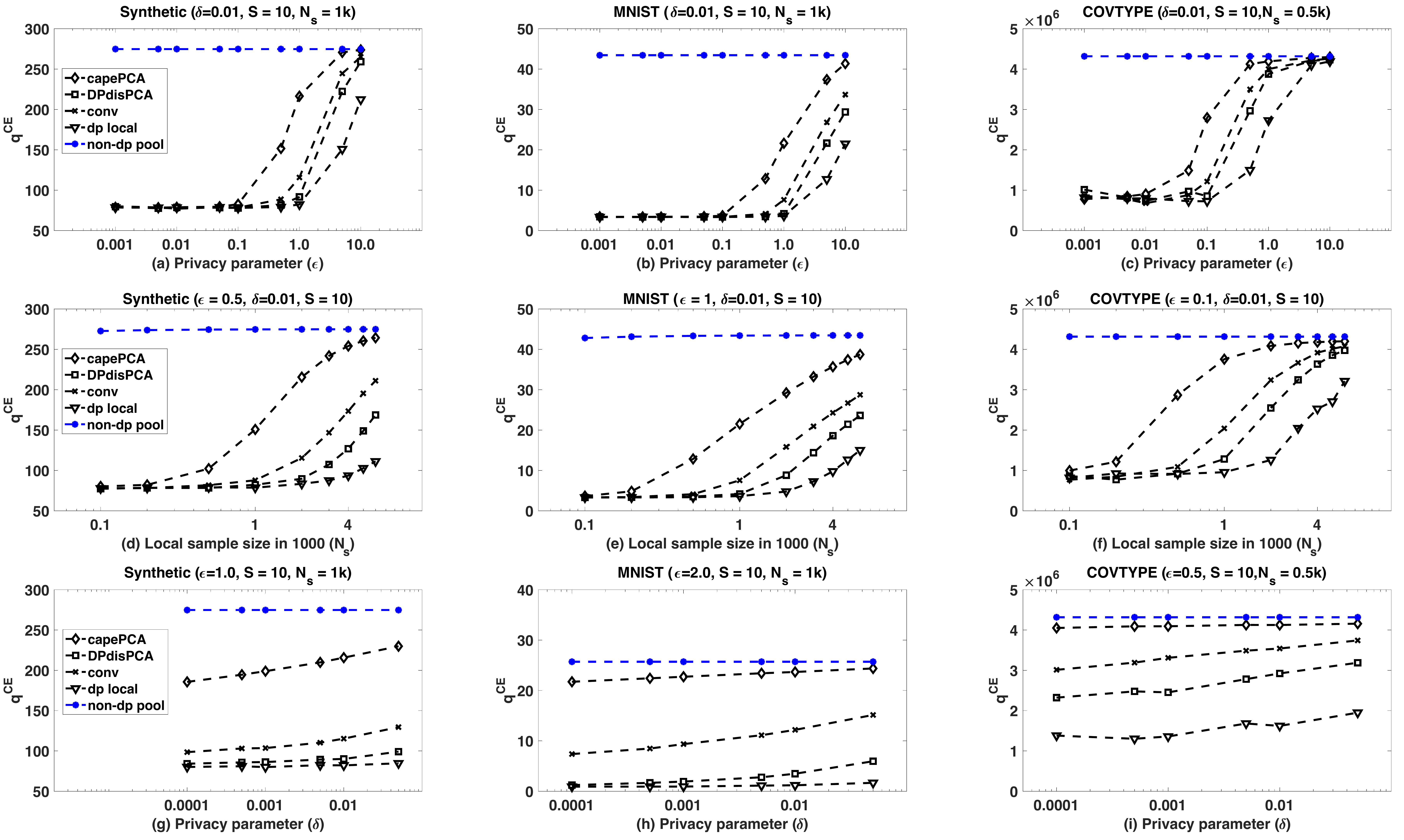}\\
  \vspace{-0.0in}
  \caption{Variation of performance in distributed PCA for synthetic and real data: (a) - (c) with privacy parameter $\epsilon$; (d) - (f) with sample size $N_s$ and (g) - (i) with privacy parameter $\delta$}
  \label{fig:all_figures_dpca}
\end{figure*}
%
%\begin{figure}[t]
%  \centering
%  % Requires \usepackage{graphicx}
%  \includegraphics[width=1\columnwidth]{images/all_figures_dpcaV3.pdf}\\
%  \vspace{-0.0in}
%  \caption{Variation of performance in distributed PCA for synthetic and real data: (a) - (c) with privacy parameter $\epsilon$; (d) - (f) with sample size $N_s$ and (g) - (i) with privacy parameter $\delta$}
%  \label{fig:all_figures_dpca}
%\end{figure}
%It is apparent that the $\capepca$ algorithm has a large parameter space. In particular, one might be interested to know how the performance of the algorithm varies as a function of privacy level and sample size. Here, we present experimental results to show empirical comparison among 
We empirically compared the proposed $\capepca$, the existing $\mathsf{DPdisPCA}$ and non-private PCA on pooled data $(\nonpriv)$. We also included the performance of differentially private PCA~\cite{dwork2014} on local data ($\local$) (i.e. data of a single site) and the conventional approach ($\conv$) (i.e. without correlated noise). We designed the experiments according to Imtiaz and Sarwate~\cite{imtiazDPCA2018}
%. We performed experiments on 
using three datasets: a \textit{synthetic} dataset ($D = 200$, $K = 50$) generated with zero mean and a pre-determined covariance matrix, the \textit{MNIST} dataset ($D = 784$, $K = 50$)~\cite{mnist} (MNIST) and the \textit{Covertype} dataset ($D = 54$, $K = 10$)~\cite{Lichman:2013} (COVTYPE). The MNIST consists of handwritten digits and has a training set of $60000$ samples, each of size $28\times 28$ pixels. The COVTYPE contains the forest cover types for $30 \times 30\ m^2$ cells obtained from US Forest Service (USFS) Region 2 Resource Information System (RIS) data. We collected the dataset from the UC Irvine KDD archive~\cite{Lichman:2013}. For our experiments, we randomly selected $60000$ samples from the COVTYPE. We preprocessed the data by subtracting the mean (centering) and normalizing the samples with the maximum $\mathcal{L}_2$ norm in each dataset to enforce the condition $\|\vect{x}_n\|_2 \le 1 \ \ \forall n$. We note that this preprocessing step is not differentially private, although it can be modified to satisfy differential-privacy at the cost of some utility. In all cases we show the average performance over 10 runs of the algorithms. As a performance measure of the produced subspace from the algorithm, we choose the captured energy: $q^\mrm{CE} = \tr(\hat{\matr{V}}^{\top} \matr{A} \hat{\matr{V}})$, where $\hat{\matr{V}}$ is the subspace estimated by an algorithm and $\matr{A}$ is the true second-moment matrix of the data. Note that, we can approximate the the captured energy in the true subspace as $\tr(\matr{V}_K(\matr{A})^\top \matr{A} \matr{V}_K(\matr{A}))$, where $\matr{A}$ is achieved from the pooled-data sample second-moment matrix and $\matr{V}_K(\matr{A})$ is achieved from the non-private PCA.

\noindent\textbf{Dependence on privacy parameter $\epsilon$.} First, we explore the trade-off between privacy and utility; i.e., between $\epsilon$ and $q^\mrm{CE}$. We note that the standard deviation of the added noise is inversely proportional to $\epsilon$ -- bigger $\epsilon$ means higher privacy risk but less noise and thus, better utility. 
%We observe this in our experiments as well. 
In Figure \ref{fig:all_figures_dpca}(a)-(c), we show the variation of $q^\mrm{CE}$ of different algorithms for different values of $\epsilon$. For this experiment, we kept the parameters $\delta$, $N_s$ and $S$ fixed. For all the datasets, we observe that as $\epsilon$ increases (higher privacy risk), the captured energy increases. The proposed $\capepca$ reaches the optimal utility for some parameter choices
%. $\capepca$ 
and clearly outperforms the existing $\mathsf{DPdisPCA}$, the $\conv$, and the $\local$ algorithms. One of the reasons that $\capepca$ outperforms $\conv$ is the smaller noise variance at the aggregator, as described \hyperlink{target:dpca_perf_gain}{before}. Moreover, $\capepca$ outperforms $\mathsf{DPdisPCA}$ because $\mathsf{DPdisPCA}$ suffers from a larger variance at the aggregator due to computation of the partial square root of $\hat{\matr{A}}_s$~\cite{imtiazDPCA2018}. However, it should be noted here that $\mathsf{DPdisPCA}$ offers a much smaller communication overhead than $\capepca$. Achieving better performance than $\local$ is intuitive because including the information from multiple sites to estimate a population parameter should always result in better performance than using the data from a single site only. An interesting observation is that for datasets with lower dimensional samples, we can use smaller $\epsilon$ (i.e., to guarantee lower privacy risk) for the same utility.

\noindent\textbf{Dependence on number of samples $N_s$.} Next, we investigate the variation in performance with sample size $N_s$. Intuitively, it should be easier to guarantee smaller privacy risk $\epsilon$ and higher utility $q^\mrm{CE}$, when the number of samples is large. Figures \ref{fig:all_figures_dpca}(d)-(f) show how $q^\mrm{CE}$ increases as a function of sample size per site $N_s$. The variation with $N_s$ reinforces the results seen earlier with variation of $\epsilon$. For a fixed $\epsilon$ and $\delta$, the utility increases as we increase $N_s$. For sufficiently large local sample size, the captured energy will reach that of $\nonpriv$. Again, we observe a sharper increase in utility for lower-dimensional dataset.

\noindent\textbf{Dependence on privacy parameter $\delta$.} Finally, we explore the variation of performance with the other privacy parameter $\delta$. Recall that $\delta$ can be considered as the probability that the algorithm releases the private information without guaranteeing privacy. We, therefore, want this to be as small as possible. However, lower $\delta$ results in larger noise variance. In Figure \ref{fig:all_figures_dpca}(g)-(i), we show how $q^\mrm{CE}$ vary with varying $\delta$. We observe that if $\delta$ is not too small, the proposed algorithm can achieve very good utility, easily outperforming the other algorithms.

\subsection{Distributed Differentially-private OTD}
\begin{figure*}[t]
  \centering
  \includegraphics[width=1\textwidth]{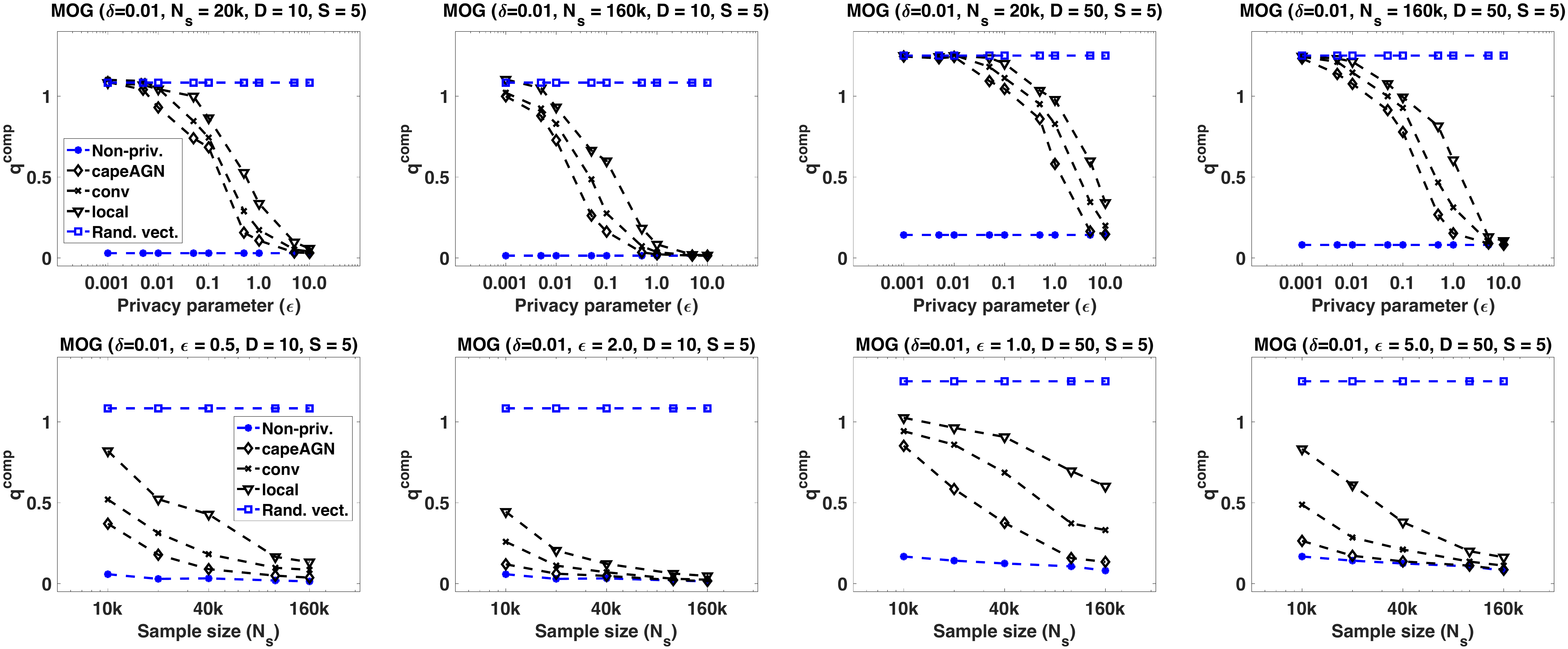}\\
  \vspace{-0.0in}
  \caption{Variation of performance in the MOG setup: top-row -- with privacy parameter $\epsilon$; bottom-row -- with sample size $N_s$}
  \label{fig:MOG_vs_all}
\end{figure*}

\begin{figure*}[t]
  \centering
  \includegraphics[width=1\textwidth]{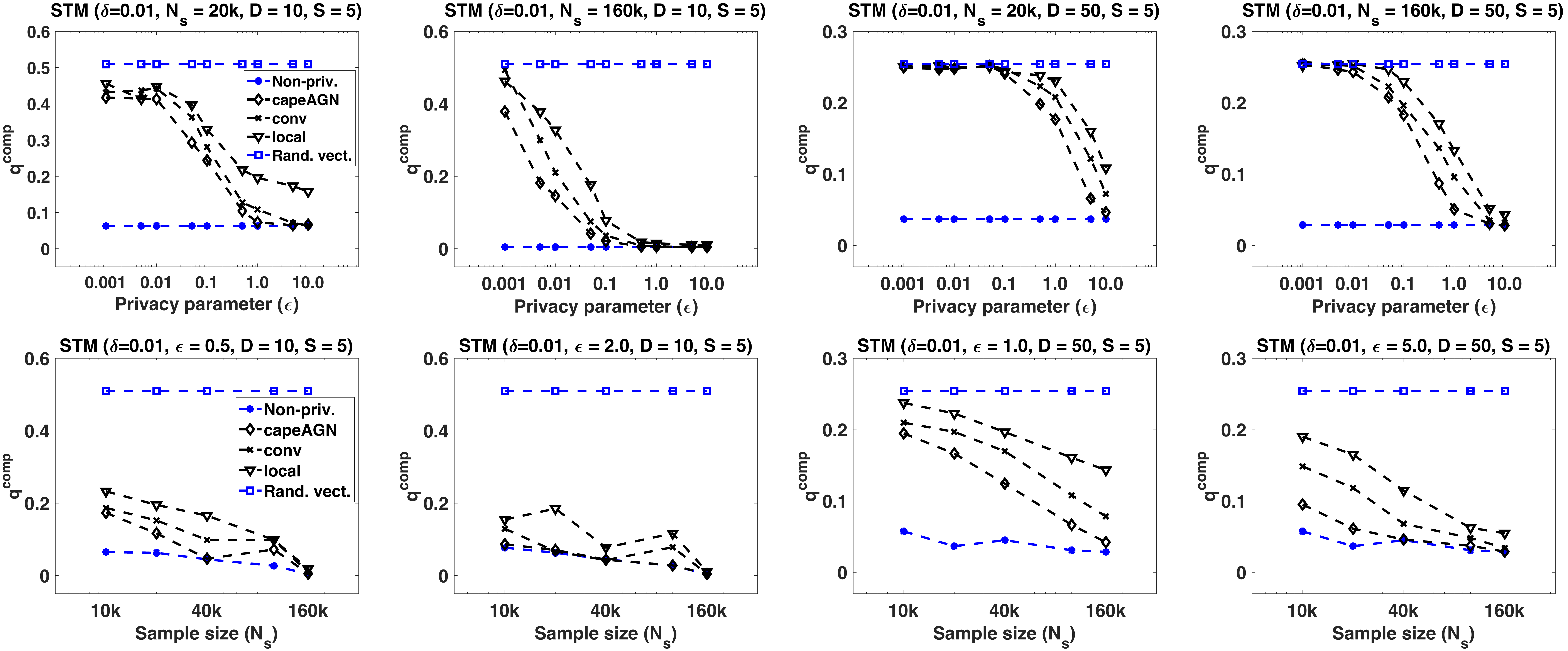}\\
  \vspace{-0.0in}
  \caption{Variation of performance in the STM setup: top-row -- with privacy parameter $\epsilon$; bottom-row -- with sample size $N_s$}
  \label{fig:STM_vs_all}
\end{figure*}

\begin{figure*}[t]
  \centering
  \includegraphics[width=1\textwidth]{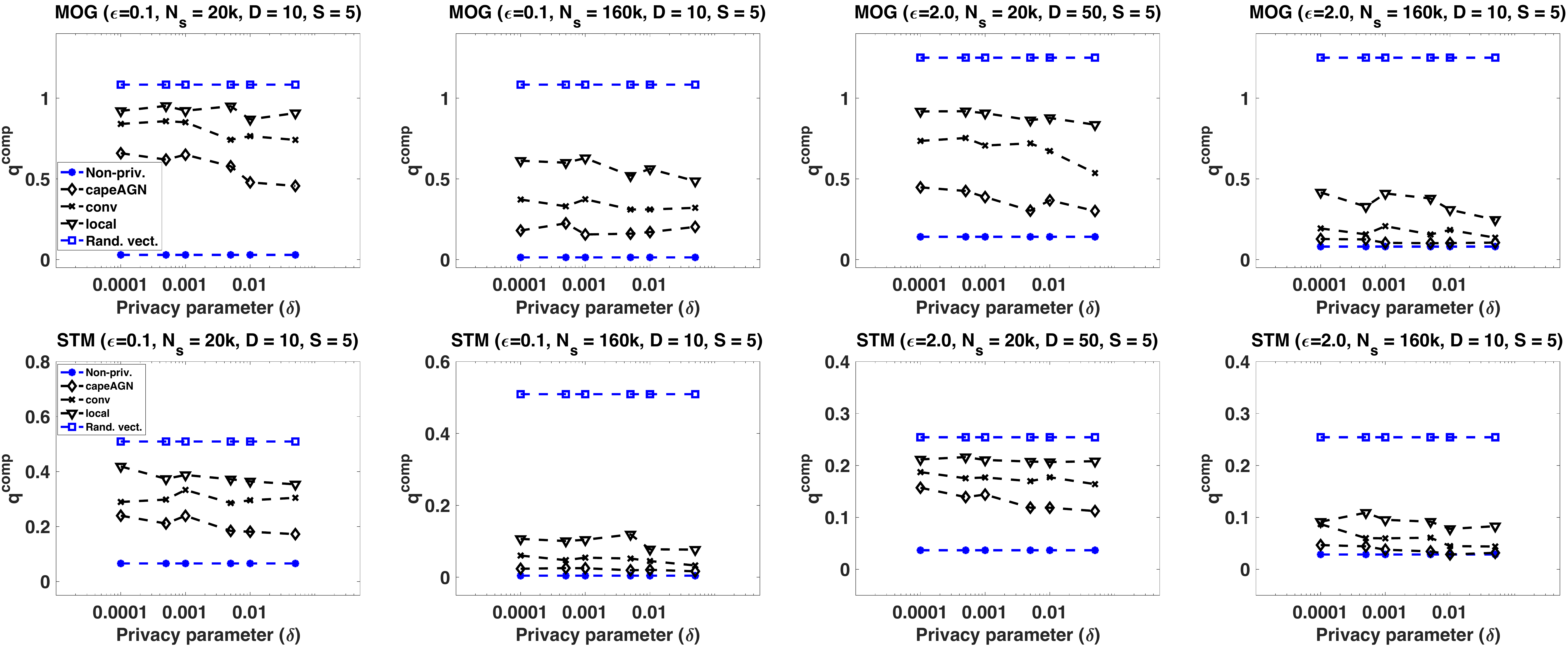}\\
  \vspace{-0.0in}
  \caption{Variation of performance with privacy parameter $\delta$: top-row -- in MOG setup; bottom-row -- in STM setup}
  \label{fig:delta_vs_all}
\end{figure*}

For the proposed $\capeagn$ algorithm, we focus on measuring how well the output of the proposed algorithm approximate the true components $\{\vect{a}_k\}$ and $\{w_k\}$. Let the recovered component vectors be $\{\hat{\vect{a}}_k\}$. We use an error metric~\cite{imtiazOTD2018} $q^\mrm{comp}$, to capture the disparity between $\{\vect{a}_k\}$ and  $\{\hat{\vect{a}}_k\}$. 
%We can compute this by proceeding as follows: 
For the $k$-th recovered component vector $\hat{\vect{a}}_k$, we compute the Euclidean distance from it to all of the true component vectors $\{\vect{a}_k\}$ and find the one with the minimum Euclidean distance. This distance is the score for the $k$-th recovered component. We take the average of all scores to get $q^\mrm{comp}$: 
\begin{align*}
q^\mrm{comp} &= \frac{1}{K}\sum_{k=1}^K ED^k_{\mrm{min}}, \mbox{ where } ED^k_\mrm{min} = \min_{k' \in [K]} \| \hat{\vect{a}}_k - \vect{a}_{k'} \|_2.
\end{align*}
A similar measure is used in dictionary learning to verify the correctness of the recovered dictionary atoms~\cite{nguyen2012}. In order for the comparison, we show the error resulting from the $\hat{\vect{a}}_k$'s achieved from the proposed $\capeagn$ algorithm, a conventional (but never proposed anywhere to the best of our knowledge) distributed differentially-private OTD algorithm that does not employ correlated noise $(\conv)$, a differentially-private OTD~\cite{imtiazOTD2018} on local data $(\local)$ and the non-private tensor power method~\cite{anandkumar2012} on pooled data $(\nonprivT)$. We also show the error considering random vectors as $\hat{\vect{a}}_k$'s $(\rand)$. The reason~\cite{imtiazOTD2018} behind showing errors pertaining to random vectors is the following: this error corresponds to the worst possible results, as we are not taking any information from data into account to estimate $\hat{\vect{a}}_k$'s. We note that, as recovering the component vectors is closely related with recovering the selection probabilities $\{w_k\}$, we only show the error associated with recovering the component vectors. We studied the dependence of $q^\mrm{comp}$ on the privacy parameters $\epsilon$, $\delta$ and the sample size $N_s$. In all cases we show the average performance over 10 runs of each algorithm. We note that the $\capeagn$ algorithm adds noise in two stages for ensuring differential-privacy: one for estimating $\matr{W}$ and another for estimating $\tens{M}_3$. We equally divided $\epsilon$ and $\delta$ to set $\epsilon_1$, $\epsilon_2$ and $\delta_1$, $\delta_2$ for the two stages. Optimal allocation of $\epsilon$ and $\delta$ in multi-stage differentially-private algorithms is still an open question.

\noindent\textbf{Performance Variation in the MOG Setup.} First, we present the performance of the aforementioned algorithms in the setting of the mixture of Gaussians. We use two \textit{synthetic data} sets of different feature dimensions ($D=10$, $K=5$ and $D=50$, $K=10$), where the common covariance is $\sigma^2 = 0.05$ and the component vectors $\{\vect{a}_k\}$ have $\mathcal{L}_2$-norm at-most 1. 

We first explore the \textit{privacy-utility tradeoff} between $\epsilon$ and $q^\mrm{comp}$. For the $\capeagn$ algorithm, the variance of the noise is inversely proportional to $\epsilon^2$ -- smaller $\epsilon$ means more noise (lower utility) and lower privacy risk. In the top-row of Figure \ref{fig:MOG_vs_all}, we show the variation of $q^\mrm{comp}$ with $\epsilon$ for a fixed $\delta = 0.01$ and $S = 5$ for two different feature dimensions, each with two different samples sizes. For both of the feature dimensions, we observe that as $\epsilon$ increases (higher privacy risk), the errors decrease and the proposed $\capeagn$ algorithm outperforms the $\conv$ and $\local$ methods. $\capeagn$ matches the performance of $\nonprivT$ method for larger $\epsilon$ values. For a particular feature dimension, we notice that if we increase $N_s$, the performance of $\capeagn$ gets even better. This is expected as the variance of the noise for $\capeagn$ is inversely proportional to square of the sample size. 

Next, we consider the performance variation with $N_s$. Intuitively, it should be easier to guarantee a smaller privacy risk for the same $\epsilon$ and a higher utility (lower error) when the number of samples is large. In the bottom row of Figure \ref{fig:MOG_vs_all}, we show how the errors vary as a function of $N_s$ for the MOG model for two different feature dimensions, while keeping $\delta = 0.01$ and $S = 5$ fixed. The variation with the sample size reinforces the results seen earlier with variation in $\epsilon$: the proposed $\capeagn$ outperforms the other algorithms under investigation for both $D=10$ and $D=50$. In general, $\capeagn$ approaches the performance of $\nonprivT$ as the sample size increases. When $\epsilon$ is large enough, the $\capeagn$ algorithm achieves as much utility as $\nonprivT$ method.

Finally, we show the variation of performance with the other privacy parameter $\delta$. Recall that $\delta$ can be interpreted as the probability that the privacy-preserving algorithm releases the private information ``out in the wild'' without any additive noise. Therefore, we want to ensure that $\delta$ is small. However, the smaller the $\delta$ is the larger the noise variance becomes. Thus smaller $\delta$ dictates loss in utility. We observe this in our experiments as well. In the top-row of Figure \ref{fig:delta_vs_all}, we show the variation of $q^\mrm{comp}$ with $\delta$ for two different feature dimensions and two different sample sizes. We kept $S=5$ fixed.  We observe that when $\epsilon$ is small, we need larger $\delta$ to achieve meaningful performance. This can be explained in the following way: for Gaussian mechanism, we need to ensure that the standard deviation of the noise $\sigma$ satisfies $\sigma \geq \frac{\Delta}{\epsilon}\sqrt{2\log\frac{1.25}{\delta}}$, where $\Delta$ is the $\mathcal{L}_2$ sensitivity of the function under consideration. This inequality can be satisfied with infinitely many $(\epsilon, \delta)$ pairs and one can keep the noise level the same for a smaller $\epsilon$ with a larger $\delta$. We observe from the figure that when both $\epsilon$ and $N_s$ are larger, the proposed $\capeagn$ can achieve very close performance to the non-private one even for very small $\delta$ values.

\noindent\textbf{Performance Variation in the STM Setup. }We performed experiments on two \textit{synthetic} datasets of different feature dimensions ($D=10$, $K=5$ and $D=50$, $K=10$) generated with pre-determined $\vect{w}$ and $\{\vect{a}_k\}$. It should be noted here that the recovery of $\{\vect{a}_k\}$ is difficult~\cite{imtiazOTD2018}, because the recovered word probabilities from the tensor decomposition, whether private or non-private, may not always be valid probability vectors (i.e., no negative entries and sum to 1). Therefore, prior to computing the $q^\mrm{comp}$, we ran a post-processing step (0-out negative entries and then normalize by summation) to ensure that the recovered vectors are valid probability vectors. This process is non-linear and potentially makes the recovery error worse. However, for practical STM, $D$ is not likely to be 10 or 50, rather it may be of the order of thousands, simulating which is a huge computational burden. In general, if we want the same privacy level for higher dimensional data, we need to increase the sample size. We refer the reader to some efficient (but non-privacy-preserving) implementations~\cite{huang2014} for tensor factorization.

As in the case of the MOG model, we first explore the \textit{privacy-utility tradeoff} between $\epsilon$ and $q^\mrm{comp}$. In the top-row of Figure \ref{fig:STM_vs_all}, we show the variation of  $q^\mrm{comp}$ with $\epsilon$ for a fixed $\delta = 0.01$ and $S = 5$ for two different feature dimensions. For both of the feature dimensions, we observe that as $\epsilon$ increases (higher privacy risk), the errors decrease. The proposed $\capeagn$ outperforms $\conv$ and $\local$ methods in all settings; and match the performance of $\nonprivT$ for large enough $\epsilon$. Increasing $N_s$ makes the proposed algorithm perform even better.

Next, in the bottom-row of Figure \ref{fig:STM_vs_all}, we show how the errors vary as a function of $N_s$ for two different feature dimensions, while keeping $\delta = 0.01$ and $S = 5$ fixed. The variation with $N_s$ reiterates the results seen earlier. The proposed $\capeagn$ outperforms all other algorithms (except the $\nonprivT$) for both $D = 10$ and $D=50$. For larger $N_s$, it achieves almost the same utility as the $\nonprivT$ algorithm. Even for smaller $\epsilon$ with a proper sample size, the error is very low. For the $D=10$ case, the $\capeagn$ always performs very closely with the $\nonprivT$ algorithm.

Lastly, we show the variation of $q^\mrm{comp}$ with $\delta$ in the bottom-row of Figure \ref{fig:delta_vs_all}. We observe similar performance trend as in the MOG setting. For smaller $\epsilon$ and sample size, we need to compensate with larger $\delta$ to achieve a performance closer to $\nonprivT$ one. However, when sample size is larger, we can get away with a smaller $\epsilon$ and $\delta$. This is intuitive as hiding one individual among a large group is easier -- the additive noise variance need not be very large and hence the performance does not take a large hit.

\section{Conclusion}\label{sec:conclusion}
In this paper, we proposed new algorithms for distributed differentially-private principal component analysis and orthogonal tensor decomposition. Our proposed algorithms achieve the same level of additive noise variance as the pooled data scenario for ensuring differential-privacy. Therefore, we attain the same utility as the differentially-private pooled data scenario in a distributed setting. This is achieved through the employment of the correlated noise design protocol, under the assumption of availability of some reasonable resources. We empirically compared the performance of the proposed algorithms with those of existing (if any) and conventional distributed algorithms on synthetic and real data sets. We varied privacy parameters and relevant dataset parameters. The proposed algorithms outperformed the existing and conventional algorithms comfortably and matched the performance of corresponding non-private algorithms for proper parameter choices. In general, the proposed algorithms offered very good utility even for strong privacy guarantees, which indicates that meaningful privacy can be attained even without loosing much utility. 

\section*{Appendix}
\appendix

\section{Algebra Related with $\cape$ Protocol}\label{appendix:cape}

\subsection{Proof of Lemma \ref{lemma:cape}}\label{appendix:cape:lemma}
\begin{algorithm}[h] 
	\caption{Correlation Assisted Private Estimation ($\cape$)\label{alg:dp_avg}}
	\begin{algorithmic}[1]
    \Require Data samples $\{\vect{x}_s\}$; privacy parameters $\epsilon$, $\delta$.
    \State Compute $\tau_s \gets \frac{1}{N_s\epsilon}\sqrt{2\log\frac{1.25}{\delta}}$
    \State At the random noise generator, generate $e_s \sim \mathcal{N}(0, \tau^2_e)$, where $\tau^2_e = (1-\frac{1}{S})\tau^2_s$ and $\sum_{s=1}^S e_s = 0$
    \State At the central aggregator, generate $f_s \sim \mathcal{N}(0, \tau^2_f)$, where $\tau^2_f = (1-\frac{1}{S})\tau^2_s$
    \For{$s = 1,\ \ldots,\ S$}
    	\State Get $e_s$ from the random noise generator
        \State Get $f_s$ from the central aggregator
        \State Generate $g_s \sim \mathcal{N}(0, \tau_g^2)$, where $\tau_g^2 = \frac{\tau_s^2}{S}$
        \State Compute and send $\hat{a}_s \gets f(\vect{x}_s) + e_s + f_s + g_s$ \label{alg:dp_avg:step_as_hat}
    \EndFor
    \State At the central aggregator, compute $a_\mrm{ag}^\mrm{imp} \gets \frac{1}{S} \sum_{s=1}^S \hat{a}_s - \frac{1}{S} \sum_{s=1}^S f_s$ \label{alg:dp_avg:step_a_ag}\\
    \Return $a_\mrm{ag}^\mrm{imp}$
    \end{algorithmic}
\end{algorithm}

\begin{lemma}\label{lemma:cape}
Let the variances of the noise terms $e_s$, $f_s$ and $g_s$ (Step \ref{alg:dp_avg:step_as_hat} of Algorithm \ref{alg:dp_avg}) be given by \eqref{eqn:cape_noise_variance}. If we denote the variance of the additive noise (for preserving privacy) in the pooled data scenario by $\tau_c^2$ and the variance of the estimator $a_\mrm{ag}^\mrm{imp}$ (Step \ref{alg:dp_avg:step_a_ag} of Algorithm \ref{alg:dp_avg}) by ${\tau_\mrm{ag}^\mrm{imp}}^2$ then Algorithm \ref{alg:dp_avg} achieves $\tau_c^2 = {\tau_\mrm{ag}^\mrm{imp}}^2$.
\end{lemma}

\begin{proof}
We recall that in the pooled data scenario, the sensitivity of the function $f(\vect{x})$ is $\frac{1}{N}$, where $\vect{x} = \left[\vect{x}_1,\ldots, \vect{x}_S\right]$. Therefore, to approximate $f(\vect{x})$ satisfying $(\epsilon, \delta)$ differential privacy, we need to have additive Gaussian noise standard deviation at least $\tau_c = \frac{1}{N\epsilon}\sqrt{2\log\frac{1.25}{\delta}}$. Next, consider the $(\epsilon, \delta)$ differentially-private release of the function $f(\vect{x}_s)$. The sensitivity of this function is $\frac{1}{N_s}$. Therefore, the $(\epsilon, \delta)$ differentially-private approximate of the function $f(\vect{x}_s)$ requires an additive Gaussian noise standard deviation at least $\tau_s = \frac{1}{N_s\epsilon}\sqrt{2\log\frac{1.25}{\delta}}$. Note that, if we assume equal number of samples in each site, then we have $\tau_c = \frac{\tau_s}{S} \implies \tau_c^2 = \frac{\tau_s^2}{S^2}$. We will now show that the $\cape$ algorithm will yield the same noise variance of the estimator at the aggregator. Recall that at the aggregator we compute $a_\mrm{ag}^\mrm{imp} = \frac{1}{S} \sum_{s=1}^S \left(\hat{a}_s - f_s\right) = \frac{1}{N} \sum_{n=1}^{N} x_n + \frac{1}{S} \sum_{s=1}^S g_s$. The variance of the estimator ${\tau_\mrm{ag}^\mrm{imp}}^2 \triangleq S\cdot \frac{\tau_g^2}{S^2} = \frac{\tau_g^2}{S} = \frac{\tau_s^2}{S^2}$, which is exactly the same as the pooled data scenario. Therefore, the $\cape$ algorithm allows us to achieve the same additive noise variance as the pooled data scenario, while satisfying at least $(\epsilon, \delta)$ differential privacy at the sites and $(\epsilon, \delta)$ differential privacy for the final output from the aggregator.
\end{proof}

\subsection{Solution of the Feasibility Problem of Section \ref{sec:weighted_dp_avg}}\label{appendix:unequal_privacy}
We formulated a feasibility problem to solve for the unknown noise variances $\{\tau_{es}^2, \tau_{gs}^2, \tau_{fs}^2\}$ as
\begin{align*}
	\underset{}{\text{minimize}} 	&\qquad 0 \\
				\text{subject to} 	&\qquad \tau_{fs}^2 + \tau_{gs}^2 \geq \tau_s^2,\  \tau_{es}^2 + \tau_{gs}^2 \geq \tau_s^2, \\
                                    			&\qquad \sum_{s=1}^S \mu_s^2 \tau_{gs}^2 = \tau_c^2,\ \sum_{s=1}^S \mu_s e_s = 0,
\end{align*}
for all $s\in [S]$, where $\{\mu_s\}$, $\tau_c$ and $\tau_s$ are known to the aggregator. As mentioned before, multiple solutions are possible. We present the details of one solution here that solves the problem with equality. \\
\noindent \textbf{Solution. }We start with $\sum_{s=1}^S \mu_s e_s = 0$. We can set
\begin{align*}
\sum_{s=1}^{S-1} \mu_s e_s + \mu_S e_S = 0 \implies \sum_{s=1}^{S-1} \mu_s^2 \tau_{es}^2 = \mu_S^2\tau_{eS}^2 \\
\implies \sum_{s=1}^{S-1} \mu_s^2 \tau_{es}^2 - \mu_S^2\tau_{eS}^2 = 0.
\end{align*}
Additionally, we have $\sum_{s=1}^{S-1} \mu_s^2 \tau_{gs}^2 + \mu_S^2\tau_{gS}^2 = \tau_c^2$. Combining these, we observe $\tau_{gS}^2 - \tau_{eS}^2 = \frac{1}{\mu_S^2} \left(\tau_c^2 - \sum_{s=1}^{S-1}\mu_s^2\tau_s^2\right)$. Moreover, for the $S$-th site, $\tau_{gS}^2 + \tau_{eS}^2 = \tau_S^2$. Therefore, we can solve for $\tau_{gS}^2$ and $\tau_{eS}^2$ as
\begin{align*}
\tau_{gS}^2 &= \frac{\tau_S^2}{2} + \frac{1}{2\mu_S^2}\left(\tau_c^2 - \sum_{s=1}^{S-1}\mu_s^2\tau_s^2\right) \\
\tau_{eS}^2 &= \frac{\tau_S^2}{2} - \frac{1}{2\mu_S^2}\left(\tau_c^2 - \sum_{s=1}^{S-1}\mu_s^2\tau_s^2\right).
\end{align*}
Additionally, we set $\tau_{fS}^2$ from $\tau_{gS}^2 + \tau_{fS}^2 = \tau_S^2$ as
\begin{align*}
\tau_{fS}^2 &= \frac{\tau_S^2}{2} - \frac{1}{2\mu_S^2}\left(\tau_c^2 - \sum_{s=1}^{S-1}\mu_s^2\tau_s^2\right).
\end{align*}
Now, we focus on setting the noise variances for $s \in [S-1]$. From the relation $\sum_{s=1}^{S-1} \mu_s^2 \tau_{es}^2 = \mu_S^2\tau_{eS}^2$, one solution is to set 
\begin{align*}
\tau_{es}^2 &= \frac{1}{\mu_s^2(S-1)}\left[\frac{\mu_S^2}{2}\tau_S^2 - \frac{1}{2}\left(\tau_c^2 - \sum_{s=1}^{S-1}\mu_s^2\tau_s^2\right)\right].
\end{align*}
Using this and $\tau_{gs}^2 = \tau_s^2 - \tau_{es}^2$, we have
\begin{align*}
\tau_{gs}^2 &=\tau_s^2 - \frac{1}{\mu_s^2(S-1)}\left[\frac{\mu_S^2}{2}\tau_S^2 - \frac{1}{2}\left(\tau_c^2 - \sum_{s=1}^{S-1}\mu_s^2\tau_s^2\right)\right].
\end{align*}
Finally, we solve for $\tau_{fs}^2 = \tau_s^2 - \tau_{gs}^2$ as
\begin{align*}
\tau_{fs}^2 &= \frac{1}{\mu_s^2(S-1)}\left[\frac{\mu_S^2}{2}\tau_S^2 - \frac{1}{2}\left(\tau_c^2 - \sum_{s=1}^{S-1}\mu_s^2\tau_s^2\right)\right].
\end{align*}
Therefore, we can solve the feasibility problem with equality.

\section{Notation and Definitions for Tensor Data}\label{appendix:tensor_preliminaries}

\textbf{Tensors} are multi-dimensional arrays, higher dimensional analogs of matrices. The very first tensor decomposition ideas (e.g. tensor rank and polyadic decomposition) are attributed to Hitchcock~\cite{hitchcock1927,hitchcock1928}. Tensor decomposition and multi-way signal models were used in the context of latent variable models in psychometrics~\cite{cattell1944}. It became popular in neuroscience, signal processing and machine learning later.

An $M$-way or $M$-th order tensor is an element of the tensor product of $M$ vector spaces. \textbf{Fibers} are higher order analogs of rows and columns. A fiber is defined by fixing every index but one. An $M$-way tensor $\tens{X} \in \mathbb{R}^{D_1 \times \ldots \times D_M}$ is rank-1 if it can be written as the outer product of $M$ vectors:
\begin{align*}
\tens{X} &= \vect{x}_1 \otimes \vect{x}_2 \otimes \ldots \otimes \vect{x}_M,	
\end{align*}
where $\vect{x}_m \in \mathbb{R}^{D_m}$ and $\otimes$ denotes the outer product. \textbf{Matricization} (or \emph{unfolding} or \emph{flattening}) is the process of reordering the elements of an $M$-way tensor into a matrix. The mode-$m$ matricization of $\tens{X} \in \mathbb{R}^{D_1 \times \ldots \times D_M}$ is denoted as $\matr{X}_{(m)}$ and is found by arranging the mode-$m$ fibers of $\tens{X}$ as the columns of $\matr{X}_{(m)}$. A \textbf{mode-$m$ product} is multiplying a tensor by a matrix in mode-$m$. Let $\tens{X} \in \mathbb{R}^{D_1 \times \ldots \times D_M}$ and $\matr{U} \in \mathbb{R}^{J \times D_m}$ then
\begin{align*}
\left[\tens{X} \times_m \matr{U}\right]_{d_1 \ldots d_{m-1}, j, d_{m+1} \ldots d_M} &= \sum_{d_m = 1}^{D_m} \left[\tens{X}\right]_{d_1 \ldots d_M} \left[ \matr{U} \right]_{j,d_m}. 
\end{align*}
We can also represent the mode-$m$ flattened tensor as
\begin{align*}
\tens{Y} = \tens{X} \times_m \matr{U} \Longleftrightarrow \matr{Y}_{(m)} = \matr{U}\matr{X}_{(m)}.
\end{align*}
The \textbf{vectorization} of the tensor $\tens{X}$ is defined as~\cite{ohlson2012,ohlson2013}
\begin{align*}
\ve\tens{X} &= \sum_{d_1=1}^{D_1} \cdots \sum_{d_M=1}^{D_M} \left[\tens{X}\right]_{d_1,\ldots,d_M}\vect{e}_{d_1}^{D_1} \circ \cdots \circ \vect{e}_{d_M}^{D_M},
\end{align*}
where $\circ$ denotes the Kronecker product~\cite{kolda2009} and $\vect{e}^{D_m}$ denotes the $D_m$-dimensional elementary (or unit basis) vector. We note that $\ve\tens{X}$ is a $(\prod_{m=1}^M D_m)$-dimensional vector. A tensor is called \textbf{symmetric} if the entries do not change under any permutation of the indices. The \textbf{rank} of a tensor $\tens{X}$ is the smallest number of rank-1 tensors that sums to the original tensor~\cite{kolda2015}. The \textbf{norm} of a tensor $\tens{X} \in \mathbb{R}^{D_1 \times \ldots \times D_M}$~\cite{anandkumar2012} is
\begin{align*}
\|\tens{X}\| &= \sqrt{\sum_{d_1=1}^{D_1}\cdots \sum_{d_M=1}^{D_M} \left[\tens{X}\right]^2_{d_1,\ldots,d_M}}.
\end{align*}
This is equivalent to the matrix Frobenius norm. We note that the norm $\|\tens{X}\|$ of a tensor $\tens{X}$ is equal to the $\mathcal{L}_2$-norm of the vectorized version of the same tensor, $\ve\tens{X}$. That is, $\|\tens{X}\| = \|\ve\tens{X}\|_2$. We also observe that for a vector $\vect{x} \in \mathbb{R}^D$, if the $\mathcal{L}_2$-norm $\|\vect{x}\|_2 = 1$ then
\begin{align*}
\|\vect{x} \otimes \cdots \otimes \vect{x}\| &= 1 \mbox{ because}\\
\left[\vect{x} \otimes \cdots \otimes \vect{x}\right]_{d_1,\ldots,d_M} &= [\vect{x}]_{d_1} \cdots [\vect{x}]_{d_M}.
\end{align*}
The \textbf{operator norm} of an $M$-way symmetric tensor $\tens{X} \in \mathbb{R}^{D \times \ldots \times D}$ is defined~\cite{anandkumar2012} as
\begin{align*}
\|\tens{X}\|_\mrm{op} &= \sup_{\|\vect{x}\|_2 = 1} \left| \tens{X}\left(\vect{x},\vect{x},\ldots,\vect{x}\right)\right|.
\end{align*}
Finally, a tensor $\tens{X} \in \mathbb{R}^{D_1 \times \ldots \times D_M}$ can be considered to be a multi-linear map~\cite{anandkumar2012} in the following sense: for a set of matrices $\{\matr{V}_m \in \mathbb{R}^{D_m \times K_m} : m = 1,2, \ldots, M\}$, the $(k_1, k_2, \ldots, k_M)$-th entry in the $M$-way tensor representation of $\tens{Z} = \tens{X}\left(\matr{V}_1, \ldots, \matr{V}_M\right) \in \mathbb{R}^{K_1 \times \ldots \times K_M}$ is
\begin{align*}
\left[\tens{Z}\right]_{k_1 \ldots k_M} &=\sum_{d_1 \ldots d_M} \left[\tens{X}\right]_{d_1 \ldots d_M} \left[\matr{V}\right]_{d_1,k_1} \cdots \left[\matr{V}\right]_{d_M,k_M}.
\end{align*}
Therefore, we have
\begin{align*}
\tens{X}\left(\matr{V}_1 \ldots \matr{V}_M\right) &= \tens{X} \times_1 \matr{V}_1^\top \cdots \times_M \matr{V}_M^\top.
\end{align*}

\section{Algebra for Various Calculations}\label{appendix:algebra}

\subsection{Calculation of $\hat{\matr{A}}$ in Section \ref{sec:dist_dppca}}\label{appendix:A_hat}
We show the calculation of  $\hat{\matr{A}}$ here in detail. \hyperlink{target:A_hat}{Recall} that the sites send their $\hat{\matr{A}}_s$ to the aggregator and the aggregator computes
\begin{align*}
\hat{\matr{A}} 	&= \frac{1}{S} \sum_{s=1}^S \left(\hat{\matr{A}}_s - \matr{F}_s\right) \\
						&= \frac{1}{S} \sum_{s=1}^S \left(\matr{A}_s + \matr{E}_s + \matr{F}_s + \matr{G}_s - \matr{F}_s\right) \\
                			&= \frac{1}{S} \sum_{s=1}^S \left(\matr{A}_s + \matr{G}_s\right) + \frac{1}{S} \sum_{s=1}^S \matr{E}_s \\
                			&= \frac{1}{S} \sum_{s=1}^S \left(\matr{A}_s + \matr{G}_s\right),
\end{align*}
where we used the relation $\sum_{s=1}^S \matr{E}_s = 0$. 

\subsection{Calculation of $\tilde{\tens{M}}_3$ in Section \ref{sec:dist_dpotd}}\label{appendix:M3_tilde}
\hypertarget{target:M3_tilde_calculation}{We} show the calculation of $\tilde{\tens{M}}_3$ here in detail. We \hyperlink{target:M3_tilde}{recall} that 
\begin{align*}
\tilde{\tens{M}}_3^s &=  \hat{\tens{M}}_3^s\left(\matr{W}, \matr{W}, \matr{W}\right).
\end{align*}
Additionally, at the aggregator, we compute 
\begin{align*}
\tilde{\tens{M}}_3 &= \frac{1}{S}\sum_{s=1}^S \left(\tilde{\tens{M}}_3^s - \tilde{\tens{F}}_3^s\right),
\end{align*}
where $\tilde{\tens{F}}_3^s = \tens{F}_3^s\left(\matr{W}, \matr{W}, \matr{W}\right)$. Then we have
\begin{align*}
	\tilde{\tens{M}}_3 &= \frac{1}{S}\sum_{s=1}^S 
		\Big(\tens{M}_3^s\left(\matr{W}, \matr{W}, \matr{W}\right) + \tens{E}_3^s\left(\matr{W}, \matr{W}, \matr{W}\right) + \\
	& \tilde{\tens{F}}_3^s + \tens{G}_3^s\left(\matr{W}, \matr{W}, \matr{W}\right) - \tilde{\tens{F}}_3^s\Big) \\
	&= \frac{1}{S}\sum_{s=1}^S \Big(\tens{M}_3^s\left(\matr{W}, \matr{W}, \matr{W}\Big) + \tens{G}_3^s\left(\matr{W}, \matr{W}, \matr{W}\right)\right) +\\
	& \left(\frac{1}{S}\sum_{s=1}^S\tens{E}_3^s\right)\left(\matr{W}, \matr{W}, \matr{W}\right) \\
	&= \frac{1}{S}\sum_{s=1}^S \left(\tens{M}_3^s\left(\matr{W}, \matr{W}, \matr{W}\right) + \tens{G}_3^s\left(\matr{W}, \matr{W}, \matr{W}\right)\right) \\
	&= \left(\frac{1}{S}\sum_{s=1}^S \tens{M}_3^s + \tens{G}_3^s\right)\left(\matr{W}, \matr{W}, \matr{W}\right),
\end{align*}
where we used the associativity of the multi-linear operation~\cite{anandkumar2012} and the relation $\sum_{s=1}^S \tens{E}_3^s = 0$.

\section{Applications of Orthogonal Tensor Decomposition \label{appendix:otd_examples}}
We review two examples from Anandkumar et al.~\cite{anandkumar2012}, which involve estimation of latent variables from observed samples. The lower-order moments obtained from the samples can be written as low-rank symmetric tensors.

\subsection{Single Topic Model (STM)}\label{sec:otd_examples_stm}
Let us consider an exchangeable bag-of-words model~\cite{anandkumar2012} for documents. Such exchangeable models can be viewed as mixture models in which there is a latent variable $h$ such that the $L$ words in the document $\vect{t}_1, \vect{t}_2, \ldots, \vect{t}_L$ are conditionally i.i.d. given $h$. Additionally, the conditional distributions are identical at all the nodes~\cite{anandkumar2012}. Let us assume that $h$ is the only topic of a given document, and it can take only $K$ distinct values. Let $D$ be the number of distinct words in the vocabulary, and $L \geq 3$ be the number of words in each document. The generative process for a document is as follows: the document's topic is drawn according to the discrete distribution specified by the probability vector $\vect{w} = \left[w_1, w_2, \ldots, w_K\right]^\top$. This is modeled as a discrete random variable $h$ such that
\begin{align*}
\Pr\left[h = k\right] &= w_k,
\end{align*}
for $k = 1, 2, \ldots, K$. Given the topic $h$, the document's $L$ words are drawn independently according to the discrete distribution specified by the probability vector $\vect{a}_h \in \mathbb{R}^D$. We represent the $L$ words in the document by $D$-dimensional random vectors $\vect{t}_l \in \mathbb{R}^D$. Specifically, if the $l$-th word is $d$, we set
\begin{align*}
\vect{t}_l &= \vect{e}_d \mbox{ for } l = 1, 2, \ldots, L,
\end{align*}
where $\vect{e}_1, \vect{e}_2, \ldots, \vect{e}_D$ are the standard coordinate basis vectors for $\mathbb{R}^D$. We observe that for any topic $k = 1,2, \ldots, K$
\begin{align*}
\mathbb{E}\left[\vect{t}_1 \otimes \vect{t}_2 | h = k\right] &= \sum_{i,j} \Pr\left[\vect{t}_1=i, \vect{t}_2=j | h = k \right] \vect{e}_i \otimes \vect{e}_j \\
							&= \mathbb{E}\left[\vect{t}_1 | h=k\right] \otimes \mathbb{E}\left[\vect{t}_2 | h=k\right] \\
                            &= \vect{a}_k \otimes \vect{a}_k.
\end{align*}
Now, we can define two moments in terms of the outer products of the probability vectors $\vect{a}_k$ and the distribution of the topics $h$
\begin{align}\label{eqn_moments}
\matr{M}_2 &= \sum_{k=1}^K w_k \vect{a}_k \otimes \vect{a}_k, \\
\tens{M}_3 &= \sum_{k=1}^K w_k \vect{a}_k \otimes \vect{a}_k \otimes \vect{a}_k.
\end{align}
The method proposed in \cite{anandkumar2012} to recover $\vect{w}$ and $\{ \vect{a}_k \}$ proceeds as follows: we observe $N$ documents. Each of the documents has number of words $L \geq 3$. The way we record what we observe is: we form an $D \times D \times D$ tensor whose $(d_1,d_2,d_3)$-th entry is the proportion of times we see a document with first word $d_1$, second word $d_2$ and third word $d_3$. 
%Let $\matr{X}_l \in \mathbb{R}^{N \times D}$ be the matrix that contains the records of the $l$-th word. What we mean is that the $d$-th column of $\matr{X}_l$ denotes the $l$-th word (in terms of $\vect{e}_1, \vect{e}_2, \ldots, \vect{e}_N$) in the $d$-th document. 
In this setting, we can estimate the moments $\matr{M}_2$ and $\tens{M}_3$, defined in (\ref{eqn_moments}), from the observed data as:
\begin{align}
\matr{M}_2 &= \mathbb{E}[\vect{t}_1 \otimes \vect{t}_2],  \\
\tens{M}_3 &= \mathbb{E}[\vect{t}_1 \otimes \vect{t}_2 \otimes \vect{t}_3].
\end{align}
We then need to perform orthogonal tensor decomposition on $\tens{M}_3$ to recover $\vect{w}$ and $\{ \vect{a}_k \}$. 

\subsection{Mixture of Gaussians (MOG)}\label{sec:otd_examples_mog}
A similar example as the single topic model is the spherical mixture of Gaussians~\cite{anandkumar2012}. Let us assume that there are $K$ components and the component mean vectors are given by the set $\{ \vect{a}_1, \vect{a}_2,\ldots,\vect{a}_K \} \subset \mathbb{R}^D$. The probability of choosing component $k$ is $w_k$. We assume that the common covariance matrix is $\sigma^2\matr{I}_D$. However, the model can be extended to incorporate different covariance matrices for different component as well~\cite{anandkumar2012,kakade2013}. The $n$-th observation of the model can be written as
\begin{align*}
\vect{t}_n &= \vect{a}_h + \vect{z},
\end{align*}
where $h$ is a discrete random variable with $\Pr[h=k] = w_k$ and $\vect{z}$ is an $D$-dimensional random vector, independent from $h$, drawn according to $\mathcal{N}(0,\sigma^2\matr{I}_D)$. Let us denote the total number of observations by $N$. Without loss of generality, we assume that $\|\vect{a}_k\|_2 \leq 1$. Now, for $D\geq K$, it has been shown~\cite{kakade2013} that if we have estimates of the second and third order moments from the observations $\vect{t}_n$ as $\matr{M}_2 = \mathbb{E}[\vect{t} \otimes \vect{t}] - \sigma^2\matr{I}_D$ and
\begin{align*}
 \tens{M}_3 = \mathbb{E}[\vect{t} \otimes \vect{t} \otimes \vect{t}] - \\
 \sigma^2\sum_{d=1}^D \left(\mathbb{E}[\vect{t}]\otimes \vect{e}_d \otimes \vect{e}_d + \vect{e}_d \otimes \mathbb{E}[\vect{t}] \otimes \vect{e}_d + \vect{e}_d \otimes \vect{e}_d \otimes \mathbb{E}[\vect{t}]\right),
\end{align*}
then these moments are decomposable as
\begin{align*}
\matr{M}_2 &= \sum_{k=1}^K w_k \vect{a}_k \otimes \vect{a}_k, \tens{M}_3 = \sum_{k=1}^K w_k \vect{a}_k \otimes \vect{a}_k \otimes \vect{a}_k.
\end{align*}

\subsection{Orthogonal Decomposition of $\tens{M}_3$}
For both the STM and the MOG model, in order to decompose $\tens{M}_3$ using the tensor power method (\ref{eqn_tensor_power_method}), we need the $\vect{a}_k$'s to be orthogonal to each other. But, in general, they are not. To employ the orthogonal tensor decomposition, we can project the tensor onto some subspace $\matr{W} \in \mathbb{R}^{D\times K}$ to ensure $\matr{W}^\top \vect{a}_k$'s are orthogonal to each other. We note that, according to the multi-linear notation, we have
\begin{align}\label{eqn_multilinear_map}
\tens{M}_3(\matr{V}_1,\matr{V}_2,\matr{V}_3) &= \sum_{k=1}^K w_k \left(\matr{V}_1^\top \vect{a}_k\right) \otimes \left(\matr{V}_2^\top \vect{a}_k\right) \otimes \left(\matr{V}_3^\top \vect{a}_k\right).
\end{align}
In order to find $\matr{W}$, we can compute the SVD($K$) on the second-order moment $\matr{M}_2 \in \mathbb{R}^{D \times D}$ as
\begin{align*}
\matr{M}_2 &= \matr{U} \matr{D} \matr{U}^\top,
\end{align*}
where $\matr{U} \in \mathbb{R}^{D \times K}$ and $\matr{D} \in \mathbb{R}^{K \times K}$. We define $\matr{W} = \matr{U} \matr{D}^{-\frac{1}{2}} \in \mathbb{R}^{D \times K}$ and then compute the projection $\tilde{\tens{M}}_3 = \tens{M}_3(\matr{W},\matr{W},\matr{W})$. We note that $\tilde{\tens{M}}_3 \in \mathbb{R}^{K \times K \times K}$ is now orthogonally decomposable. We use the tensor power iteration (\ref{eqn_tensor_power_method}) on $\tilde{\tens{M}}_3$ to recover the weights $\{w_k\}$ and the component vectors $\{ \vect{a}_k \}$. The detail of the tensor power method can be found in Anandkumar et al.~\cite{anandkumar2012}.

\section{Differentially-private OTD}\label{appendix:dpotd}
We note that the key step in the orthogonal tensor decomposition algorithm is the mapping given by (\ref{eqn_tensor_power_method}). In order to ensure differential privacy for the orthogonal decomposition, we may either add noise at each iteration step scaled to the $\mathcal{L}_2$ sensitivity~\cite{dwork2014} of the operation given by (\ref{eqn_tensor_power_method}) or we can add noise to the tensor $\tens{X}$ itself just once. Adding noise in each iteration step might result in a poor utility/accuracy of the recovered eigenvectors and eigenvalues. We intend to add noise to the tensor itself prior to employing the tensor power method. In the following, we are showing the sensitivity calculations for the pooled data scenario. Extension to the distributed case is straightforward (replacing $N$ with $N_s$).

First, we focus on the exchangeable single topic model setup that we described in Appendix \ref{sec:otd_examples_stm}. We observe and record $N$ documents. Let us consider two sets of documents, which differ in only one sample (e.g., the last one). Let the empirical second-order moment matrices be $\matr{M}_2$ and $\matr{M}'_2$ and the third-order moment tensors be $\tens{M}_3$ and $\tens{M}'_3$, respectively, for these two sets. We consider the two tensors, $\tens{M}_3$ and $\tens{M}'_3$, as \emph{neighboring}. We observe that
\begin{align*}
\matr{M}_2 	&= \frac{1}{N} \sum_{n=1}^N \vect{t}_{1,n} \vect{t}_{2,n}^\top \\
            			&= \frac{1}{N} \sum_{n=1}^{N-1} \vect{t}_{1,n} \vect{t}_{2,n}^\top + \frac{1}{N} \vect{t}_{1,N} \vect{t}_{2,N}^\top, \\
\matr{M'}_2 	&= \frac{1}{N} \sum_{n=1}^{N-1} \vect{t}_{1,n} \vect{t}_{2,n}^\top + \frac{1}{N} \vect{t}'_{1,N} \vect{t'}_{2,N}^\top,         
\end{align*}
where $\vect{t}_{l,n}$ denotes the $l$-th word of the $n$-th document. Similarly, we observe
\begin{align*}
\tens{M}_3 	&= \frac{1}{N} \sum_{d=1}^D \vect{t}_{1,n} \otimes \vect{t}_{2,n} \otimes \vect{t}_{3,n}\\
            &= \frac{1}{N} \sum_{n=1}^{N-1} \vect{t}_{1,n} \otimes \vect{t}_{2,n} \otimes \vect{t}_{3,n} + \frac{1}{N} \vect{t}_{1,N} \otimes \vect{t}_{2,N} \otimes \vect{t}_{3,N}, \\
\tens{M'}_3 	&= \frac{1}{N} \sum_{n=1}^{N-1} \vect{t}_{1,n} \otimes \vect{t}_{2,n} \otimes \vect{t}_{3,n} + \frac{1}{N} \vect{t}'_{1,N} \otimes \vect{t}'_{2,N} \otimes \vect{t}'_{3,N}.
\end{align*}
As mentioned before, we perform the SVD on $\matr{M}_2$ first to compute $\matr{W}$. We intend to use the $\ag$ algorithm~\cite{dwork2014} to make this operation differentially private. We look at the following \hypertarget{target:sensitivity_M2}{quantity}:
\begin{align*}
\|\matr{M}_2-\matr{M'}_2\|_2 = \frac{1}{N} \|\vect{t}_{1,N} \vect{t}_{2,N}^\top - \vect{t'}_{1,N} \vect{t'}_{2,N}^\top\|_2 \\
= \frac{1}{N} \sup_{\|\vect{u}\|_2, \|\vect{v}\|_2 = 1} \Big\{\vect{u}^\top \left(\vect{t}_{1,N} \vect{t}_{2,N}^\top - \vect{t'}_{1,N} \vect{t'}_{2,N}^\top\right)\vect{v}\Big\} \\
\leq \frac{\sqrt{2}}{N} = \Delta_{2,S},
\end{align*}
because of the encoding $\vect{t}_{l,n} = \vect{e}_d$. For the mixture of Gaussians model, we note that we assumed $\| \vect{a}_k \|_2 \leq 1$ for all $k \in \{1, 2, \ldots, K\}$. To find a bound on $\|\matr{M}_2-\matr{M'}_2\|_2$, we consider the following: for identifiability of the $\{\vect{a}_k\}$, we have to assume that the $\vect{a}_k$'s are linearly independent. In other words, we are interested in finding the directions of the components specified by $\{\vect{a}_k\}$. In that sense, while obtaining the samples, we can divide the samples by a constant $\zeta$ such that $\|\vect{t}_n\|_2 \leq 1$ is satisfied. From the resulting second- and third-order moments, we will be able to recover $\{\vect{a}_k\}$ up to a scale factor. It is easy to show using the definition of largest eigenvalue of a symmetric matrix~\cite{boyd2004} that
\begin{align*}
\|\matr{M}_2-\matr{M'}_2\|_2 &= \frac{1}{N} \sup_{\|\vect{u}\|_2 = 1} \Big\{\vect{u}^\top \left(\vect{t}_N \vect{t}_N^\top - \vect{t'}_N \vect{t'}_N^\top\right)\vect{u}\Big\} \\
& = \frac{1}{N} \sup_{\|\vect{u}\|_2 = 1} \Big\{ \left|\vect{u}^\top\vect{t}_N\right|^2-\left|\vect{u}^\top\vect{t'}_N\right|^2 \Big\}\\
& \leq \frac{1}{N} = \Delta_{2,M},
\end{align*}
where the inequality follows from the relation $\|\vect{t}_n\|_2 \leq 1$. We note that the largest singular value of a matrix is the square root of the largest eigenvalue of that matrix. For the distributed case, as \hyperlink{orig:sensitivity_M2}{mentioned} before, the sensitivity of $\matr{M}_2^s$ depends only on the local sample size. We can therefore use the $\ag$ algorithm~\cite{dwork2014} (i.e., adding Gaussian noise with variance scaled to $\Delta_{2,S}$ or $\Delta_{2,S}$ to $\matr{M}_2$) to make the computation of $\matr{W}$ $(\epsilon_1,\delta_1)$-differentially private.

Now, we focus on the tensor $\tens{M}_3$. We need to project $\tens{M}_3$ on $\matr{W}$ before using the tensor power method. We can choose between making the projection operation differentially private, or we can make the $\tens{M}_3$ itself differentially private before projection. We found that making the projection operation differentially private involves addition of a large amount of noise and more importantly, the variance of the noise to be added depends on the alphabet size (or feature dimension) $D$ and the singular values of $\matr{M}_2$. Therefore, we choose to make the tensor itself differentially private. 
We are interested to find the sensitivity of the tensor valued function $f(\tens{M}_3) = \tens{M}_3$, which is simply the identity map. That is, we need to find the maximum quantity that this function can change if we replace the argument $\tens{M}_3$ with a neighboring $\tens{M'}_3$. For our exchangeable single topic model setup, \hypertarget{target:sensitivity_M3}{we have}
\begin{align*}
\|\tens{M}_3 - \tens{M'}_3\| &= \Big\|\frac{1}{N} \vect{t}_{1,N} \otimes \vect{t}_{2,N} \otimes \vect{t}_{3,N} - \frac{1}{N} \vect{t'}_{1,N} \otimes \vect{t'}_{2,N} \otimes \vect{t'}_{3,N} \Big\| \\
& \leq \frac{\sqrt{2}}{N} = \Delta_{3,S}, 
\end{align*}
because only one entry in the $D \times D \times D$ tensor $\vect{t}_{1,N} \otimes \vect{t}_{2,N} \otimes \vect{t}_{3,N}$ is non-zero (in fact, the only non-zero entry is 1). Now, for the mixture of Gaussians model, we define
\begin{align*}
\tens{T} &= \sigma^2\sum_{d=1}^D \left(\mathbb{E}[\vect{t}]\otimes \vect{e}_d \otimes \vect{e}_d + \vect{e}_d \otimes \mathbb{E}[\vect{t}] \otimes \vect{e}_d + \vect{e}_d \otimes \vect{e}_d \otimes \mathbb{E}[\vect{t}]\right) % \\
%&= \sigma^2\sum_{d=1}^D \left(\left(\frac{1}{N}\sum_{n=1}^N \vect{t}_n\right)\otimes \vect{e}_d \otimes \vect{e}_d + \vect{e}_d \otimes \left(\frac{1}{N}\sum_{n=1}^N \vect{t}_n\right) \otimes \vect{e}_d + \vect{e}_d \otimes \vect{e}_d \otimes \left(\frac{1}{N}\sum_{n=1}^N \vect{t}_n\right)\right)
\end{align*}
Therefore, we have
\begin{align*}
\tens{T} - \tens{T}' &= \frac{\sigma^2}{N}\sum_{d=1}^D \Big(\left(\vect{t}_N-\vect{t}'_N\right)\otimes \vect{e}_d \otimes \vect{e}_d + \\
& \vect{e}_d \otimes \left(\vect{t}_N-\vect{t}'_N\right) \otimes \vect{e}_d + \vect{e}_d \otimes \vect{e}_d \otimes \left(\vect{t}_N-\vect{t}'_N\right)\Big) \\
\|\tens{T} - \tens{T}'\| &\leq \frac{3D\sigma^2}{N}\|\vect{t}_N - \vect{t}'_N\|_2 \\
& \leq \frac{6D\sigma^2}{N},
\end{align*}
where the last inequality follows from $\|\vect{t}_n\|_2 \leq 1$. Now, we have
\begin{align*}
\|\tens{M}_3 - \tens{M'}_3\| &= \Big\|\frac{1}{N} \vect{t}_{N} \otimes \vect{t}_{N} \otimes \vect{t}_{N} - \\
& \frac{1}{N} \vect{t'}_{N} \otimes \vect{t'}_{N} \otimes \vect{t'}_{N} + \tens{T} - \tens{T}'\Big\| \\
& \leq \frac{2}{N} + \frac{6D\sigma^2}{N} = \Delta_{3,M}, 
\end{align*}
because $\|\vect{t}_N \otimes \vect{t}_N \otimes \vect{t}_N\| = 1$ in our setup. \hyperlink{orig:sensitivity_M3}{Again}, we note that in the distributed setting, the sensitivity of the local $\tens{M}_3^s$ depends only on the local sample size. In the following, we present the two recently proposed algorithms from Imtiaz and Sarwate~\cite{imtiazOTD2018}. The first one uses a symmetric tensor made with i.i.d. entries from a Gaussian distribution, while the second proposed method uses a symmetric tensor made with entries taken from a sample vector drawn from an appropriate distribution. Both of the algorithms guarantee $(\epsilon,\delta)$-differential privacy.

\subsection{Addition of i.i.d. Gaussian Noise}\label{sec:agn}
\begin{algorithm}[t] 
	\caption{$\agn$ Algorithm \label{alg:agn}}
	\begin{algorithmic}[1]
    \Require Sample second-order moment matrix $\matr{M}_2 \in \mathbb{R}^{D\times D}$ and third-order moment tensor $\tens{M}_3 \in \mathbb{R}^{D \times D \times D}$, privacy parameters $\epsilon_1$, $\epsilon_2$, $\delta_1$, $\delta_2$.
    \State Generate $D \times D$ symmetric matrix $\matr{E}$ with $\{E_{ij} : i \in [D], j \leq i\}$ drawn i.i.d. from $\mathcal{N}(0,\tau_1^2)$ and $E_{ij} = E_{ji}$. Here, $\tau_1 = \frac{\Delta_2}{\epsilon_1} \sqrt{2\log\left(\frac{1.25}{\delta_1}\right)}$
    \State $\hat{\matr{M}}_2 \gets \matr{M}_2 + \matr{E}$
    \State Compute SVD$(K)$ on $\hat{\matr{M}}_2 = \matr{U}\matr{D}\matr{U}^\top$ and find $\matr{W} = \matr{U}\matr{D}^{-\frac{1}{2}}$
    \State Draw a sample vector $\vect{b} \in \mathbb{R}^{D_\mrm{sym}}$ whose entries are i.i.d $\sim \mathcal{N}(0,\tau_2^2)$, where $D_\mrm{sym}={D+2 \choose 3}$ and $\tau = \frac{\Delta_3}{\epsilon_2} \sqrt{2\log\left(\frac{1.25}{\delta_2}\right)}$
    \State Generate a symmetric tensor $\tens{E} \in \mathbb{R}^{D\times D \times D}$ from the entries of $\vect{b}$
    \State Compute $\hat{\tens{M}}_3 \gets \tens{M}_3 + \tens{E}$
    \State Compute $\tilde{\tens{M}}_{3} \gets \hat{\tens{M}}_3(\matr{W},\matr{W},\matr{W})$\\
    \Return The differentially private orthogonally decomposable tensor $\tilde{\tens{M}}_{3} $, projection subspace $\matr{W}$
    \end{algorithmic}
\end{algorithm}

The $\agn$ algorithm first uses the $\ag$ algorithm~\cite{dwork2014} to compute a differentially-private estimate of the second-order moment matrix $\matr{M}_2$. We note that the $\mathcal{L}_2$ sensitivity of $\matr{M}_2$ is given by $\Delta_{2,S}$ and $\Delta_{2,M}$ in Appendix \ref{appendix:dpotd} for the STM and MOG models, respectively. Therefore, we generate a $D \times D$ symmetric matrix $\matr{E}$ whose upper triangle and diagonal entries are sampled i.i.d. from $\mathcal{N}(0,\tau_1^2)$ and lower triangle entries are copied from upper triangle. Here, $\tau_1 = \frac{\Delta_{2,S}}{\epsilon_1} \sqrt{2\log\left(\frac{1.25}{\delta_1}\right)}$ for the STM and $\tau_1 = \frac{\Delta_{2,M}}{\epsilon_1} \sqrt{2\log\left(\frac{1.25}{\delta_1}\right)}$ for the MOG model. By computing the SVD($K$) on the $(\epsilon_1,\delta_1)$-differentially private estimate of $\matr{M}_2$ (denoted $\hat{\matr{M}}_2$), we find the subspace $\matr{W}$ required for whitening and also for recovering the component vectors $\{\vect{a}_k\}$. Next, we draw a $D_{\mrm{sym}}$-dimensional vector $\vect{b}$ with i.i.d. entries from $\mathcal{N}(0,\tau_2^2)$, where $D_{\mrm{sym}}={D+2 \choose 3}$ and $\tau_2 = \frac{\Delta_{3,S}}{\epsilon_2} \sqrt{2\log\left(\frac{1.25}{\delta_2}\right)}$ for the STM and $\tau_2 = \frac{\Delta_{3,M}}{\epsilon_2} \sqrt{2\log\left(\frac{1.25}{\delta_2}\right)}$ for the MOG model. In order to preserve the symmetry of the third-order tensor $\tens{M}_3$ upon noise addition, we form a symmetric tensor $\tens{E} \in \mathbb{R}^{D\times D \times D}$ from the entries of $\vect{b}$. This noise tensor is then added to $\tens{M}_3$ to achieve $\hat{\tens{M}}_3$, which is the $(\epsilon_2,\delta_2)$-differentially private estimate of $\tens{M}_3$. Finally, we project $\hat{\tens{M}}_3$ on the subspace $\matr{W}$ to get the orthogonally decomposable tensor $\tilde{\tens{M}}_{3}$. The detailed procedure is depicted in Algorithm \ref{alg:agn}.

\begin{theorem}[Privacy of $\agn$ Algorithm]
Algorithm \ref{alg:agn} computes an $(\epsilon_1+\epsilon_2,\delta_1+\delta_2)$-differentially private orthogonally decomposable tensor $\tilde{\tens{M}}_3$.
\end{theorem}

\begin{proof}
From the Gaussian mechanism~\cite{dwork2006,dwork2014}, we know that if the $\mathcal{L}_2$ sensitivity of a vector valued function $f$ is denoted by $\Delta f$, then adding independently drawn random noise distributed as $\mathcal{N}(0, \tau^2)$ ensures $(\epsilon,\delta)$-differential privacy, where
\begin{align*}
\tau &= \frac{\Delta f}{\epsilon} \sqrt{2\log\left(\frac{1.25}{\delta}\right)}.
\end{align*}
Now, in order to make the function $f(\tens{M}_3) = \tens{M}_3$ differentially private, we need to find the $\mathcal{L}_2$ sensitivity of $f(\tens{M}_3)$. We computed the sensitivity of this function in Appendix \ref{appendix:dpotd}. That is
\begin{align*}
\|\tens{M}_3 - \tens{M'}_3\| &\leq \Delta_{3,S} \mbox{ for STM, and } \\
\|\tens{M}_3 - \tens{M'}_3\| &\leq \Delta_{3,M} \mbox{ for MOG}.
\end{align*}
We need to generate a symmetric tensor of the same dimension as $\tens{M}_3$ with i.i.d. entries from the distribution $\mathcal{N}(0, \tau^2)$, where $\Delta f = \Delta_{3,S}$ (or $\Delta_{3,M}$) and $(\epsilon, \delta) = (\epsilon_2, \delta_2)$. We note that a $D$-dimensional $M$-mode symmetric tensor is fully determined by
\begin{align}\label{eqn_unique_elements}
D_\mrm{sym} &= {D+M-1 \choose M}
\end{align}
elements~\cite{comon2008}. The computation of $\hat{\tens{M}}_3$ is $(\epsilon_2,\delta_2)$-differentially private. We project $\hat{\tens{M}}_3$ onto the subspace $\matr{W}$ to get the orthogonally decomposable tensor $\tilde{\tens{M}}_3$. We recall that we compute $\matr{W}$ from $\hat{\matr{M}}_2$, which is the $(\epsilon_1, \delta_1)$ differentially private approximate to $\matr{M}_2$. The computation of $\tilde{\tens{M}}_3$ for recovering the weights $\{w_k\}$ and $\{\vect{a}_k\}$ is therefore $(\epsilon_1+\epsilon_2,\delta_1+\delta_2)$-differentially private. The overall algorithm is shown in Algorithm \ref{alg:agn}. The above method can be considered as a tensor-analogue of the Analyze Gauss method for symmetric matrices~\cite{dwork2014}.
\end{proof}

\subsection{Addition of Vector Noise}\label{sec:avn}
\begin{algorithm}[t] 
	\caption{$\avn$ Algorithm \label{alg:avn}}
	\begin{algorithmic}[1]
    \Require Sample second-order moment matrix $\matr{M}_2 \in \mathbb{R}^{D\times D}$ and third-order moment tensor $\tens{M}_3 \in \mathbb{R}^{D \times D \times D}$, privacy parameters $\epsilon_1$, $\epsilon_2$, $\delta_1$, $\delta_2$.
    \State Generate $D \times D$ symmetric matrix $\matr{E}$ with $\{E_{ij} : i \in [D], j \leq i\}$ drawn i.i.d. from $\mathcal{N}(0,\tau_1^2)$ and $E_{ij} = E_{ji}$. Here, $\tau_1 = \frac{\Delta_2}{\epsilon_1} \sqrt{2\log\left(\frac{1.25}{\delta_1 + \delta_2}\right)}$
    \State $\hat{\matr{M}}_2 \gets \matr{M}_2 + \matr{E}$
    \State Compute SVD$(K)$ on $\hat{\matr{M}}_2 = \matr{U}\matr{D}\matr{U}^\top$ and find $\matr{W} = \matr{U}\matr{D}^{-\frac{1}{2}}$ \label{alg:avn:svdk}
    \State Draw a sample vector $\vect{b} \in \mathbb{R}^{D_{\mrm{sym}}}$ from the density given by (\ref{eqn_noise_vector}), where $D_{\mrm{sym}}={D+2 \choose 3}$ and $\beta = \frac{\epsilon_2}{\Delta_3}$
    \State Generate a symmetric tensor $\tens{E} \in \mathbb{R}^{D\times D \times D}$ from the entries of $\vect{b}$
    \State Compute $\hat{\tens{M}}_3 \gets \tens{M}_3 + \tens{E}$
    \State Compute $\tilde{\tens{M}}_{3} \gets \hat{\tens{M}}_3(\matr{W},\matr{W},\matr{W})$\\
    \Return The differentially private orthogonally decomposable tensor $\tilde{\tens{M}}_{3} $, projection subspace $\matr{W}$
    \end{algorithmic}
\end{algorithm}

The $\avn$ algorithm first uses the $\ag$ algorithm~\cite{dwork2014} to compute a $(\epsilon_1, \delta)$ differentially-private estimate of the second-order moment matrix $\matr{M}_2$ and then computes the subspace $\matr{W}$ required for whitening and also for recovering the component vectors of $\tens{M}_3$. Next, we draw a $D_\mrm{sym}$-dimensional vector $\vect{b}$ from the density~\cite{anand2011}:
\begin{align}\label{eqn_noise_vector}
f_b(\vect{b}) &= \frac{1}{\alpha} \exp \left( -\beta \|\vect{b}\|_2\right),
\end{align}
where $\alpha$ is a normalizing constant and $\beta$ is a parameter of the density. Later we will choose appropriate values for $\beta$ to ensure desired privacy levels. In order to preserve the symmetry of the third-order tensor $\tens{M}_3$ upon noise addition, we form a symmetric tensor $\tens{E} \in \mathbb{R}^{D\times D \times D}$ from the entries of $\vect{b}$. This noise tensor is then added to $\tens{M}_3$ to achieve $\hat{\tens{M}}_3$. Finally, we project $\hat{\tens{M}}_3$ on the subspace $\matr{W}$ to get the orthogonally decomposable tensor $\tilde{\tens{M}}_3$. The detailed procedure is shown in Algorithm \ref{alg:avn}.

\begin{theorem}[Privacy of $\avn$ Algorithm]
Algorithm \ref{alg:avn} computes an $(\epsilon_1+\epsilon_2,\delta)$-differentially private orthogonally decomposable tensor $\tilde{\tens{M}}_3$.
\end{theorem}

\begin{proof}
In order to make the function $f(\tens{M}_3) = \tens{M}_3$ differentially private, we consider the algorithm
\begin{align*}
\tens{Y} &= \tens{M}_3 + \tens{E},
\end{align*}
where $\tens{E}$ is a symmetric tensor of the same dimension as $\tens{M}_3$. We note that $\tens{E}$ consists $D_{\mrm{sym}}={D+2 \choose 3}$ number of unique entries. We propose to draw a vector $\vect{b} \in \mathbb{R}^{D_{\mrm{sym}}}$ according to the density in (\ref{eqn_noise_vector}) and then form a symmetric tensor $\tens{E}$ from the entries of $\vect{b}$. The probability of the event of drawing a particular sample from (\ref{eqn_noise_vector}) is the same as drawing a symmetric tensor with the same unique entries as the aforementioned vector from some equivalent density on symmetric tensors. Now, we are interested in the ratio of the density of $\tens{Y}$ under $\tens{M}_3$ and the density of $\tens{Y}$ under $\tens{M'}_3$
\begin{align*}
\frac{f(\tens{Y}|\tens{M}_3)}{f(\tens{Y}|\tens{M'}_3)} &= \frac{f_b(\tens{Y}-\tens{M}_3)}{f_b(\tens{Y}-\tens{M}_3)} \\
&= \frac{\exp \left( -\beta \|\ve \tens{Y} - \ve \tens{M}_3\|_2\right)}{\exp \left( -\beta \|\ve \tens{Y} - \ve \tens{M'}_3\|_2\right)} \\
&\leq \exp \left( \beta \|\ve \tens{M'}_3 - \ve \tens{M}_3\|_2\right) \\
&\leq \exp \left( \beta \|\tens{M'}_3 - \tens{M}_3\|\right) \\
&\leq \exp \left( \beta \Delta_3\right),
\end{align*}
where the inequality is introduced following from the triangle inequality of norms. Therefore, we observe that if we set $\beta = \frac{\epsilon_2}{\Delta_{3}}$, the algorithm $\tens{Y} = \tens{M}_3 + \tens{E}$ becomes $(\epsilon_2,0)$-differentially private. We set $\beta = \frac{\epsilon_2}{\Delta_{3,S}}$ for the STM and $\beta = \frac{\epsilon_2}{\Delta_{3,M}}$ for the MOG. We project the output of the algorithm onto $\matr{W}$ to obtain $\tilde{\tens{M}}_3$. The full algorithm is shown in Algorithm \ref{alg:avn} and is $(\epsilon_1+\epsilon_2,\delta)$-differentially private.
\end{proof}

We note here that we do not need to specify the normalizing constant $\alpha$ in (\ref{eqn_noise_vector}). This is because sampling from this distribution can be performed without any knowledge of $\alpha$. What we do to sample from the density (\ref{eqn_noise_vector}) is the following: we have to sample a radius and a direction. The direction we can pick uniformly by sampling $D_{\mrm{sym}}$-dimensional standard Gaussian vector with i.i.d. entries and normalizing it. The radius is Erlang distributed with parameters $(D_{\mrm{sym}},\beta)$. We can generate this by taking the sum of $D_{\mrm{sym}}$ exponential variables with parameter $\beta$. Note that the $\agn$ and $\avn$ algorithms essentially differ in one step -- the density from which the noise vector $\vect{b}$ is drawn from. However, the implications are further-reaching. With $\avn$, the computation of $\hat{\tens{M}}_3$ is pure $\epsilon_2$-DP. Therefore, if one uses an $\epsilon_1$-DP algorithm for Step \ref{alg:avn:svdk} in Algorithm \ref{alg:avn}, or if the tensor is already orthogonally decomposable (i.e., no need for whitening), then the $\avn$ algorithm would provide a pure $\epsilon$-DP algorithm for OTD.

\bibliography{refs.bib}
\bibliographystyle{IEEEtran}

\end{document}